\documentclass{article}

\usepackage[pdftex]{graphicx}
\usepackage[table,dvipsnames]{xcolor}
\usepackage[round]{natbib}
\usepackage[export]{adjustbox}[2011/08/13]

\usepackage{microtype}
\usepackage{subfigure}
\usepackage{booktabs} 

\usepackage{amsmath}
\usepackage{amssymb}
\usepackage{bm}
\usepackage{multirow}

\usepackage[algo2e,ruled]{algorithm2e}

\usepackage[colorlinks]{hyperref}

\hypersetup{citecolor=[rgb]{0.45,0.0,1.0},linkcolor=Periwinkle}

\usepackage{hyperref}

\usepackage[accepted]{icml2020}

\usepackage{bbold}

\usepackage{amsmath,amssymb,amsfonts,amsthm}
\newtheorem{theorem}{Theorem}

\theoremstyle{definition}
\newtheorem{definition}{Definition}
\newtheorem{example}{Example}[section]

\usepackage{stmaryrd}
\usepackage{graphicx}
\usepackage{mathpartir}

\usepackage[font=small,labelfont=bf]{caption}

\usepackage{hyperref}

\newcommand \anonymize [1] {#1}

\usepackage{tabu}
\usepackage{color}

\usepackage{tikz}
\usetikzlibrary{bayesnet}
\usetikzlibrary{positioning}
\usepackage{tikzsymbols}

\usepackage{listings}
\usepackage{courier}
\usepackage[T1]{fontenc}  

\definecolor{deepblue}{rgb}{0,0,0.5}
\definecolor{deepred}{rgb}{0.6,0,0}
\definecolor{deepgreen}{rgb}{0,0.5,0}
\newcommand \argmin {\operatornamewithlimits{argmin}}

\newcommand \tab {{\phantom o}}
\newcommand \from \leftarrow
\newcommand \ubar \underline
\newcommand \lift \widehat
\newcommand \oftype {{\hspace{0.08em}:\hspace{0.08em}}}
\newcommand \subs \mapsto 
\newcommand \liftedindex [2] {#1\hspace{0.05em}{\scriptsize\pmb[}\hspace{0.1em}#2\hspace{0.2em}{\scriptsize\pmb]}}
\newcommand \redto {\,\Rightarrow\,}
\newcommand \Mid {{\;\bigm|\;}}
\newcommand \sample {{\operatorname*{sample}}}
\newcommand \observe {{\operatorname*{observe}}}
\newcommand \op \operatorname
\newcommand \Sum {\textstyle\sum\limits}
\newcommand \Prod {\textstyle\prod\limits}
\newcommand{\precision}{\Lambda}

\usepackage{colortbl}
\definecolor{orange}{rgb}{1,0.647,0}
\definecolor{darkred}{rgb}{0.5,0,0}

\newcommand{\OO}{\mathcal{O}}
\newcommand{\NN}{\mathcal{N}}

\icmltitlerunning{Functional Tensors for Probabilistic Programming}

\begin{document}

\twocolumn[
\icmltitle{Functional Tensors for Probabilistic Programming}



\icmlsetsymbol{equal}{*}
\icmlsetsymbol{atuber}{$\dagger$}

\begin{icmlauthorlist}
\icmlauthor{Fritz Obermeyer}{uber,equal}
\icmlauthor{Eli Bingham}{uber,equal}
\icmlauthor{Martin Jankowiak}{uber,equal}
\icmlauthor{Du Phan}{uber}
\icmlauthor{Jonathan P. Chen}{facebook,atuber}
\end{icmlauthorlist}

\icmlaffiliation{uber}{Uber AI, San Francisco, CA, USA}

\icmlaffiliation{facebook}{Facebook, Menlo Park, CA, USA}

\icmlcorrespondingauthor{Fritz Obermeyer}{fritzo@uber.com}

\icmlkeywords{probabilistic programming, variable elimination, bayesian inference}

\vskip 0.3in
]

\printAffiliationsAndNotice{\icmlEqualContribution. \textsuperscript{$\dagger$}Work done at Uber AI.} 

\begin{abstract}
It is a significant challenge to design probabilistic programming systems that can accommodate a wide variety of inference strategies within a unified framework.
Noting that the versatility of modern automatic differentiation frameworks is based in large part on the unifying concept of tensors, we describe a software abstraction for integration---functional tensors---that captures many of the benefits of tensors, while also being able to describe continuous probability distributions. 
Moreover, functional tensors are a natural candidate for generalized variable elimination and parallel-scan filtering algorithms that enable parallel exact inference for a large family of tractable modeling motifs.
We demonstrate the versatility of functional tensors by integrating them into the modeling frontend and inference backend of the Pyro programming language.
In experiments we show that the resulting framework enables a large variety of inference strategies, including those that mix exact and approximate inference.%
\end{abstract}

\section{Introduction}


Probabilistic programming systems allow specification of probabilistic models in high-level programming languages and provide automation of probabilistic inference \citep{van2018introduction}.
It remains a significant challenge to design unified frameworks that can accommodate a wide variety of exact and approximate inference strategies, including message-passing, MCMC, and variational inference.
This work is motivated by the general goal of enabling mixed inferences strategies for probabilistic programs.
As a concrete example consider an inference algorithm that combines modern black-box variational inference with classic algorithms that leverage conjugacy (e.g.~Kalman filters).
Enabling the former requires support for Monte Carlo sampling and automatic differentiation (AD), while the latter calls for symbolic computation of sums (for discrete factors) and integrals (for Gaussian factors).
Recent work \citep{obermeyer2019tensor} exploits the algebraic properties of tensors to support such mixed inference in discrete latent variable models.
We extend these ideas with \emph{functional tensors},
a software abstraction and embedded domain-specific language for seminumerical \emph{automatic integration}
that serves as an intermediate representation for a wide variety of mixed inference strategies in probabilistic programming systems. 




\section{Functional tensors} \label{sec:semantics-all}

Tensors, or more properly ``multidimensional arrays,'' are a popular and versatile software abstraction for performing parallelizable operations on homogeneous blocks of memory.
Each tensor is backed by a single block of memory addressable by a tuple of bounded integers, where each integer indexes into a \emph{dimension} of the tensor.
Tensor libraries provide operations that act on tensors, including pointwise operations like addition and multiplication, reduction operations such as product and sum, and combined operations such as matrix multiplication and convolution.
A important property of tensor libraries is support for \emph{broadcasting}, whereby an operation defined on smaller tensors or scalars can be uniquely extended to an operation on tensors with extra dimensions on the left, so long as shapes are compatible \citep{oliphant2006guide}.

This property is extensively exploited by the Pyro probabilistic programming language \citep{bingham2018pyro}
and its implementation of tensor variable elimination for exact inference in discrete latent variable models \citep{obermeyer2019tensor},
in which each random variable in a model is associated with a distinct tensor dimension
and broadcasting is used to compile a probabilistic program into a discrete factor graph \citep{obermeyer2018automated}.

Functional tensors (hereafter ``{\bf funsors}'') both formalize and extend this seemingly idiosyncratic but highly successful approach to probabilistic program compilation
by generalizing tensors and broadcasting to allow free variables of non-integer types that appear in probabilistic models, such as real number, real-valued vector, or real-valued matrix.
Building on this, we describe a simple language of lazy funsor expressions that can serve as a unified intermediate representation for a wide variety of probabilistic programs and inference algorithms.
While in general there is no finite representation of functions of real variables, we provide a funsor interface for restricted classes of functions,\footnote{Note that funsor expressions technically represent \emph{distributions} in the functional analysis sense, since expressions containing Delta terms are not proper functions of their free variables and are only formally defined via integration against test functions.} including lazy algebraic expressions, non-normalized Gaussian functions, and Dirac delta distributions.
This restricted class of funsors retains the important property of tensors that an atomic funsor with $n$ free variables is backed by $O(1)$ many blocks of memory, and operations on funsors can be implemented by $O(1)$ many parallel operations (e.g.~GPU kernels) on that memory.

In the remainder of this section, we overview funsor syntax (Sec.~\ref{sec:syntax}), type judgements (Sec.~\ref{sec:shapes}), and operational semantics (Sec.~\ref{sec:patterns}), including the atomic distribution funsors Tensor, Gaussian, and Delta; 
we also include a complete example (Fig.~\ref{ex:syntax}) and describe how inference algorithms may be implemented as nonstandard interpretations (Sec.~\ref{sec:nonstandard}).

\subsection{Funsor syntax} \label{sec:syntax}

Funsors are terms in a first order language of arrays and array indices; we exclude higher order functions.

\begin{definition} \label{def:type}
  A \emph{type} is defined by the grammar
  \begin{align*}
    \tau \in \op{Type} ::=\;
      & \mathbb Z_n
       && \text{``bounded integer''} \\
      \Mid & \mathbb Z_{n_1}{\times}\cdots{\times}\mathbb Z_{n_k} \!\to \mathbb R
       && \text{``real-valued array''}
  \end{align*}
  for any $n,n_1,\dots,n_k,k \in \mathbb N$. A \emph{type context} is a set $\Gamma=(v_1\oftype\tau_1,\dots,v_k\oftype\tau_k)$ of name:type pairs for names $v\in\mathbb S$ in a countable set of symbols $\mathbb S$ (e.g.~strings).
  We write real-valued array types as $\mathbb R^{n_1 \times \cdots \times n_k}$ and scalars $\mathbb R^1$ as $\mathbb R$.
\end{definition}
Note a \emph{type} generalizes the size of a tensor dimension, and a \emph{type context} generalizes the shape of a tensor.
A \emph{funsor} generalizes both tensors and lazy tensor expressions.
\begin{definition} \label{def:funsor}
  A \emph{funsor} is defined by the grammar
  \begin{align*}
    e \in \op{Funsor} ::= \; & \op{Tensor}(\Gamma, w) && \text{``discrete factor''}\\
      \Mid & \op{Gaussian}(\Gamma, i, \precision) && \text{``Gaussian factor''}\\
      \Mid & \op{Delta}(v, e) && \text{``point mass''}\\
      \Mid & \op{Variable}(v, \tau) && \text{``delayed value''}\\[-0.1em]
      \Mid & \lift f(e_1, \dots, e_n) && \text{``apply function''}\\
      \Mid & e_1[v \subs e_2] && \text{``substitute''}\\
      \Mid & \Sum_v e && \text{``marginalize''}\\[-0.2em]
      \Mid & \Prod_v e  && \text{``plated product''}
  \end{align*}
  where
  $\Gamma$ is a type context, $w,i,\precision$ are multidimensional arrays of numerical data,
  $v \in \mathbb S$ is a variable name, $\tau\in T$ is a type, and $f$ is any function defined on numerical objects, e.g.~binary multiplication $\times$
  and constants $0$ and $1$.
\end{definition}


\begin{figure}
    \input{gmm_example}
\end{figure}

\subsection{Typing rules} \label{sec:shapes}

Funsor expressions are typed according to the types of their subexpressions and free variables.\footnote{See Appendix~\ref{sec:rules.types} for complete list of typing rules.}
Definition \ref{def:fv} first introduces some notation used here and elsewhere in the paper.
\begin{definition} \label{def:fv}
  The set of free variables of a funsor $e$ is denoted $\op{fv}(e)$.
  A funsor is \emph{open} if it has free variables and \emph{closed} otherwise.
  Each basic numerical object ${x\in\tau}$ defines a \emph{ground} funsor ${\widehat x = \op{Tensor}((), x)}$.
  To declare that $\Gamma$ is a type context for the free variables of a funsor $f$ of type $\tau$, we write $\Gamma \vdash f\oftype\tau$. As sets, type contexts are unordered.\footnote{In our implementation, to ensure determinism and reproducibility we fix a canonical order based on orders of subexpressions, but this does not affect our semantics.}
\end{definition}

Some terms may introduce or eliminate free variables.
For example, the marginalization and plated product funsor operations eliminate an input variable.\footnote{We sometimes overload this notation to multiple variables}
\begin{minipage}{0.48\columnwidth}
    \begin{mathpar}
        \inferrule{
          \Gamma, v \oftype \tau \vdash e \oftype \mathbb R^s
        }{
          \Gamma \vdash \Sum_v e : \mathbb R^s
        }{}
    \end{mathpar}
\end{minipage}
\begin{minipage}{0.48\columnwidth}
\begin{mathpar}
    \inferrule{
      \Gamma, v \oftype \tau \vdash e \oftype \mathbb R^s
    }{
      \Gamma \vdash \Prod_v e : \mathbb R^s
    }{}
\end{mathpar}
\end{minipage}

Most other funsor operations, such as binary product $\times$, 
simply aggregate the free variables of all of their arguments:
\begin{mathpar}
    \inferrule{
      f \in \tau_1\times\dots\times\tau_n \to \tau_0 \\\\
      \Gamma_1 \vdash e_1 \oftype \tau_1 \\
      \cdots \\
      \Gamma_n \vdash e_n \oftype \tau_n
    }{
      \cup_n \Gamma_n \vdash \lift f(e_1, \dots, e_n) : \tau_0
    }{}
\end{mathpar}

\subsection{Operational semantics} \label{sec:patterns}

Funsor computations are executed by seminumerical term rewriting.
We specify an \emph{interpretation} in the form of a complete set of patterns and rewrite rules%
\footnote{See Appendix~\ref{sec:rules.rewriting} for a partial list of rewrite rules.}
for every term in the language \ref{def:funsor} and rely on a dispatch mechanism to match and execute rules until termination \citep{carette2009finally}, using types to ensure rewrites are valid.
Rewrites include both symbolic term rewriting (including reflection, in which an expression is rewritten to a lazy version of itself) and low-level numerical computation, similar to AD systems.

The low-level numerical computations are performed in rewrite rules that manipulate distribution funsors,
which are the basic latent factors in sum-product expressions constructed during probabilistic inference.
We focus attention on three special distributions that are jointly closed under products and marginalization,%
\footnote{\label{footnote:closure}The \textsc{Exact} interpretation disallows marginalizing over Gaussian mixture components; the approximate \textsc{MomentMatching} interpretation allows arbitrary marginalization.
}
 and thus especially attractive as representations of latent variable models.
These three funsors are:
\emph{i}) Tensor funsors to represent discrete joint probability mass functions;
\emph{ii}) Gaussian funsors to represent joint multivariate normal distributions among sets of real-tensor valued variables, possibly dependent on other discrete variables;
and \emph{iii}) Delta funsors to represent degenerate distributions and Monte Carlo samples.

\subsubsection{Tensor}

Tensor funsors represent a non-normalized mass function as a single tensor (multidimensional array) of weights.
Memory cost and computation cost are both exponential in the number of free variables.
The crucial rewrite rule for Tensor funsors allows operations $\widehat f(e_1, \dots, e_n)$ on Tensor funsors $e_1,\dots,e_n$ to be eagerly evaluated even in the presence of free variables; this is especially useful when e.g.~$f$ is a neural network whose inputs depend on lazily sampled discrete random variables:
\begin{mathpar}
  \widehat f\left(\op{Tensor}(\Gamma_1, w_1), \dots, \op{Tensor}(\Gamma_n, w_n)\right)
    \redto \op{Tensor}(\cup_k \Gamma_k, f\left(w_1, \dots, w_n)\right)
\end{mathpar}
This rule admits an efficient and general implementation when using a modern tensor
library like PyTorch as a backend, since all of their ops natively support broadcasting.

\subsubsection{Gaussian}

Gaussian funsors represent a Gaussian density function among multiple real-array-valued free variables.
using the information form of the Kalman filter \citep{anderson1979optimal,bar1995multitarget}, i.e.~as pair $(i,\precision)$, where $i=\precision\mu$ is the information vector, $\mu$ is the mean, and $\precision=\Sigma^{-1}$ is the precision matrix, the inverse of the covariance matrix $\Sigma$.
The information form is useful in information fusion problems because it allows representation of rank-deficient joint distributions, such as a conditional distribution treated as a single Gaussian factor%
.
We implement marginalization via Cholesky decomposition in the usual way, thus restricting marginalization to variables with full-rank precision matrices%
.
Binary and plated products of Gaussian funsors correspond to Bayesian fusion in which the information vectors and precision matrices are added:
\begin{mathpar}
  \op{Gaussian}(\Gamma_1, i_1, \precision_1) \times \op{Gaussian}(\Gamma_2, i_2, \precision_2)
    \redto \op{Gaussian}(\Gamma_1 \cup \Gamma_2, i_1 + i_2, \precision_1 + \precision_2)
\end{mathpar}
where the underlying numerical arrays $i_k, \precision_k$ have been aligned correctly based on the information in the type contexts and operations on them like $+$ broadcast.\footnote{See Appendix \ref{sec:rules} for details on atomic term data handling}
Gaussians are canonicalized to map the zero vector to zero log density%
.
Normalized Gaussian densities are represented as lazy products of a Tensor funsor (for the normalization constant) and a Gaussian funsor (for geometry),
so while Gaussian funsors are not closed under marginalization, products of Tensor and Gaussian funsors are.
Memory cost is quadratic and computation cost is cubic in the total number of elements in all free real-tensor-valued variables; both costs are exponential in the number of bounded integer free variables. 

\subsubsection{Delta}

Gaussian and Tensor funsors are jointly closed under marginalization and products,
but they can only represent a limited set of probability distributions over real variables.
Rather than adding other non-Gaussian real-valued density terms to our language,
none of which have such favorable algebraic properties,
we introduce a representation of empirical distributions,
so that even analytically intractable joint distributions can be represented approximately.
Delta funsors represent a point distribution as a pair of numerical arrays $(v, x)$, where $v$ is a symbol and $x$ is a Tensor funsor, possibly with free discrete variables corresponding to batch dimensions.
The crucial rewrite rule for Delta funsors triggers substitution in binary products: if $v\in\op{fv}(e_2)$,
$$
  \op{Delta}(v,e_1)\times e_2 \redto \op{Delta}(v,e_1)\times e_2[v\subs e_1]
$$
Delta funsors are normalized to integrate to $1$: if $v\notin\op{fv}(e_3)$ (for example, as is the case after applying the previous rule),
$$
    \Sum_v \op{Delta}(v,e_1) \times e_3 \redto e_3
$$
We can combine these two rules with the Tensor broadcasting rule to represent Monte Carlo expectations $\sum_v \op{Delta}(v, e) \times \widehat f(\op{Variable}(v\oftype\mathbb R))$ for arbitrary $f$, significantly expanding the expressiveness of the language.

Memory and computational costs for a product of Deltas are linear in the number of terms in the product.

\subsection{Approximation with nonstandard interpretation} \label{sec:nonstandard}

Tensors, Gaussians, and Deltas are algebraically closed in combination, i.e. any sum-product of Tensor, Gaussian, and Delta factors can be rewritten%
\textsuperscript{\ref{footnote:closure}}
to a product of zero or more Deltas, an optional Tensor, and an optional Gaussian.
Our rewrite system captures this fact as a normal form funsor%
\footnote{See Appendix~\ref{sec:rules.rewriting} for details.}
representing a lazy finitary product, together with rules for commutativity, associativity, distributivity, and substitution.

This closure property makes it possible to specify and compose program transformations as nonstandard interpretations \citep{carette2009finally}.
Each interpretation is a set of rewrite rules that conform to the typing rules in this section and the appendix, and users can choose and interleave interpretations at runtime.
For example an \textsc{Exact} interpretation eagerly evaluates tractable funsors but leaves non-analytic integrals lazy,
a fully \textsc{Lazy} interpretation records an expression for optimization and static analysis,
and an \textsc{Optimize} interpretation (discussed in Sec.~\ref{sec:ve}) implements variable elimination by using distributivity to rewrite lazy sum-product funsor expressions.
Most interpretations build on \textsc{Exact} and \textsc{Lazy},
sharing all of their rewrite rules except for a small number of new or modified rules.

Program transformations are especially important in our language because, 
unlike typical tensor operations in AD libraries, some funsor expressions may be non-analytic or computationally intractable, in which case they can only be evaluated approximately.
We can express and compose even these non-semantics-preserving transformations as interpretations, provided they preserve this closure property;
for example, \textsc{MonteCarlo} and \textsc{MomentMatching} interpretations (discussed in Secs.~\ref{sec:monte-carlo} and \ref{sec:moment-matching} respectively) add extra rules for approximate integration through nonstandard interpretations of atomic funsors and marginalization operations.
As we will see, the hierarchy of computational complexity in the three distribution funsors makes this language a natural substrate for approximations where intractable exact terms are rewritten to tractable approximate terms.

\begin{figure*}[t!]
\begin{minipage}{1\textwidth}
    \begin{center}
\renewcommand \arraystretch {1.2}
\begin{tabular}{@{\hspace{0.2em}} r @{\hspace{1em}} l | l}
1 & \textbf{fun} GenerativeModel($x$)
  & $p \from 1$ \\
  2 & \quad $z \from \sample(\op{Gaussian}((v\oftype\mathbb R),i_z,\precision_z))$
  & $p \from p \times \op{Gaussian}((v\oftype\mathbb R),i_z,\precision_z)[v \subs \op{Variable}(z\oftype\mathbb R)]$ \\
3 & \quad $y \from \exp(z)$ \\
4 & \quad $\observe(\op{Gaussian}((v\oftype\mathbb R,\theta\oftype\mathbb R),i_x,\precision_x)[\theta \subs y],\, x)$
  & $p \from p \times \op{Gaussian}((v\oftype\mathbb R,\theta\oftype\mathbb R),i_x,\precision_x)[\theta \subs y, v \subs x]$ \\
5 & \textbf{end}
  & Maximize over $i_x,\Lambda_x,i_z,\Lambda_z$: $\sum\limits_z p$
\end{tabular}
\renewcommand \arraystretch 1
\end{center}
\caption{
User-facing probabilistic program (left) and automatic inference (right)
for maximum marginal likelihood inference. \texttt{sample} and \texttt{observe} statements multiply the joint probability by their input funsor; \texttt{sample} statements also return Variables for lazy evaluation.
} \label{fig:map-code}

\end{minipage}
\end{figure*}

\begin{figure*}[t!]
    \begin{center}
\renewcommand \arraystretch {1.2}
  \begin{tabular}{@{\hspace{0.2em}} r @{\hspace{1em}} l | l}
    1 & \textbf{fun} GenerativeModel($x$)
      & $p \from 1$ \\
    2 & \quad $z \from \sample(\op{Gaussian}((v\oftype\mathbb R),i_z, \precision_z))$
      & $p \from p \times \op{Gaussian}((v\oftype\mathbb R),i_z,\precision_z)[v \subs \op{Variable}(z\oftype\mathbb R)]$ \\
    3 & \quad $\observe(\op{Gaussian}((v\oftype\mathbb R,\theta\oftype\mathbb R),i_x,\precision_x)[\theta \subs z],\, x)$
      & $p \from p \times \op{Gaussian}((v\oftype\mathbb R,\theta\oftype\mathbb R),i_x,\precision_x)[v \subs x, \theta \subs z]$ \\
    4 & \textbf{end} \\
    5 & \textbf{fun} InferenceModel($x$)
      & $q \from 1$ \\
    6 & \quad $z \from \sample(\op{Gaussian}((v\oftype\mathbb R,\theta\oftype\mathbb R),i_q, \precision_q)[\theta \subs x])$
      & $q \from q \times \op{Gaussian}((v\oftype\mathbb R,\theta\oftype\mathbb R),i_q, \precision_q)[v\subs z, \theta \subs x]$ \\
    7 & \textbf{end}
      & Maximize over $i_q,\Lambda_q$: $\sum\limits_z q \log \frac p q$
  \end{tabular}
\renewcommand \arraystretch 1
\end{center}
\caption{
    User-facing probabilistic program (left) and automatic inference (right) for variational inference with delayed sampling; \texttt{sample} and \texttt{observe} statements behave as in Fig.~\ref{fig:map-code}.
    The argument $x$ is a Tensor funsor of observations.
  The quantity maximized is the ELBO, demonstrating that other computations in approximate inference are easy to express using our marginalization term.
} \label{fig:elbo-code}

\end{figure*}

\section{An intermediate language for probabilistic programming} \label{sec:ppl}

Probabilistic programming languages like Pyro extend general-purpose languages with two new primitives for representing probabilistic models as programs, a \texttt{sample} statement for sampling random variables and an \texttt{observe} statement for conditioning program executions on data.
Funsors fill two roles in probabilistic programming:
as representations of lazy tensor expressions in user-facing model code generated by nonstandard interpretation of \texttt{sample} statements to return Variable funsors, 
and as representations of joint distributions as products of factors in automatic inference strategies.
We demonstrate these roles in two probabilistic inference tasks, shown in Figs.~\ref{fig:map-code} and \ref{fig:elbo-code}.\footnote{See Appendix \ref{sec:appexamples} for detailed walkthrough of these examples}

Note that Figs.~\ref{fig:map-code}-\ref{fig:elbo-code} use Gaussian factors for concreteness but if we change the factors to Tensors we can \emph{reuse the same modeling code}, and the inference backend compiles to \emph{the same funsor code} except for the atomic terms, demonstrating the strength of the language for building reusable inference components.
For example, the \texttt{sample} statements return lazy values by default, but by interleaving the \textsc{Exact} and \textsc{Lazy} interpretations we can easily implement a sophisticated inference strategy like delayed sampling \citep{murray2017delayed} in which fragments of a program are lazily evaluated, marginalized with variable elimination and used as low-variance proposal distributions in a particle filter.

We also remark that the simplicity of this design comes with other tradeoffs; our language as described only supports a limited subset of probabilistic programs, 
filling a niche analogous to that of the intermediate expression languages in popular trace-based automatic differentiation systems like PyTorch \citep{paszke2017automatic}, 
in which the class of differentiable programs in the host language is restricted.



\section{Algorithms employing funsors} \label{sec:algorithms}


We now describe several algorithms for performing exact and approximate inference,
implemented as syntax extensions or nonstandard interpretations for our term language.
None of these algorithms is novel in isolation,
but they demonstrate the utility of our language's closure properties
and its suitability for building extensible, general-purpose inference tools.
All interpretations are detailed in App.~\ref{sec:rules}.

\subsection{Funsor variable elimination} \label{sec:ve}

Variable elimination is a dynamic programming algorithm for performing exact inference efficiently by exploiting conditional independence to distribute sums inside of products.
To perform variable elimination with funsors, we record a lazy funsor sum-product expression with the \textsc{Lazy} interpretation, rewrite the expression with an \textsc{Optimize} interpretation using a standard library for tensor contraction \citep{smith2018opt_einsum}, and evaluate the optimized expression with the \textsc{Exact} interpretation.
We further implement plated variable elimination following the algorithm of \citep{obermeyer2019tensor} nearly verbatim,
but generalizing from discrete to arbitrary free variable types because the \textsc{Optimize} interpretation is agnostic to the atomic term types.

\subsection{Parallel-scan Bayesian filtering} \label{sec:parallel-scan}

Parallel-scan Bayesian filtering \citep{sarkka2019temporal} offers an
exponential speedup of sequentially structured variable elimination problems on parallel hardware such as GPUs.
We provide a general implementation of this class of algorithms that integrates seamlessly with the variable elimination machinery of the previous section,
first adding new syntax for a generalized ``Markov product'' operation, and then implementing a parallel-scan rewrite rule in the \textsc{Exact} interpretation for this operation.
This Markov product operation and rule cover familiar cases including Hidden Markov Models (as Markov products of $\op{Tensor}$s) and Kalman filtering (as inference in Markov products of $\op{Gaussian}$s); however the general syntax supports filtering in arbitrarily complex sequential models.
Moreover, the implementation consists of a series of batched products, marginalizations and substitutions, making it automatically compatible with other interpretations that modify those rewrite rules.
See Appendix~\ref{sec:markov-product} for details.

\subsection{Alternate semirings} \label{sec:adjoint}

The previous two sections 
describe how funsors enable high-performance, concise, model-agnostic implementations of parallel variable elimination algorithms that immediately generalize to new atomic term types.
In fact, while variable elimination in its most well-known form computes marginal likelihoods, the same generic algorithm can be applied to many other problems by replacing the (sum, product) operations with different semirings \citep{kohlas2008semiring,belle2016semiring,khamis2016faq}, such as (max,product) for computing
MAP (maximum a posteriori) estimates.
We exploit this fact to make our implementations immediately reusable for these other computational tasks, relying on nonstandard interpretation of the marginalization and binary and plated product operations.

\subsection{Moment matching approximation} \label{sec:moment-matching}

Variable elimination provides a tractable exact inference algorithm in structured probabilistic models with either all discrete or all Gaussian factors.
However exact inference becomes exponentially expensive in models combining both discrete and Gaussian factors, e.g.~the switching linear dynamical system in Sec.~\ref{sec:slds}.

To enable tractable inference in structured probabilistic models combining discrete and Gaussian factors, we implement an approximate \emph{moment matching} interpretation generalizing Interacting Multiple Model (IMM) filters \citep{mazor1998interacting} and similar to expectation propagation \citep{minka2001expectation}.
This interpretation adds a new rewrite rule whereby a Gaussian mixture
is approximated by a single joint Gaussian of matching normalizer, mean, and covariance:
\begin{mathpar}  
    \Sum_v \op{Tensor}(\Gamma_1, w) \times \op{Gaussian}(\Gamma_2, i, \precision)
    \redto \op{Gaussian}(\Gamma_2 - (v\oftype\tau), i^\prime, \precision^\prime) \times \Sum_v \op{Tensor}(\Gamma_1, w^\prime)
\end{mathpar}
where $i^\prime$, $\precision^\prime$, and $w^\prime$ are computed to match moments of the original expression. For example, in the simplest case of a discrete mixture of Gaussian factors indexed by $k$
and with weights $w_k$ this rewrite rule results in a $\op{Gaussian}$ with mean and covariance given by
\begin{equation}
\nonumber
\begin{split}
\mu^\prime &= \sum_k w_k \mu_k  = \precision^{\prime -1} i^\prime\\
 \Sigma^\prime  &= \sum_k w_k \Sigma_k + \sum_k w_k (\mu^\prime  - \mu_k)(\mu^\prime  -\mu_k)^{\rm T}  = \precision^{\prime -1}
\end{split}
\end{equation}

\subsection{Differentiable Monte Carlo approximation} \label{sec:monte-carlo}

Above we have described a variety of funsor algorithms that can be used to exactly compute or
approximate marginal likelihoods, variational objectives, and other key quantities in Bayesian inference.
We can combine these with a \textsc{MonteCarlo} interpretation that approximates some Tensor or Gaussian funsors in a computationally or analytically intractable part of a larger sum-product expression with Monte Carlo samples in Delta funsors, while preserving the tractable parts of the expression for exact computation.
Importantly, by carefully designing rewrite rules using reparametrized samplers \citep{kingma2013auto} and DiCE factors \citep{foerster2018dice} where appropriate, we ensure that the \textsc{MonteCarlo} interpretation preserves differentiability to all orders; see Appendix~\ref{app:monte-carlo} for details.

\section{Related work}

\citet{murray2017delayed} introduce \emph{delayed sampling}, a programmatic approach to structured Rao-Blackwellization that combines eager and lazy sampling.
\citet{obermeyer2019tensor} generalize discrete variable elimination to factor graphs with plates, enabling fast inference on parallel hardware.
Our work combines these approaches, extending vectorized and mixed eager/lazy inference to continuous models
and formalizing the resulting expressions as \emph{open terms}, i.e.~terms with free variables \citep{barendregt2013lambda}.

PSI Solver \citep{gehr2016psi} and Hakaru \citep{narayanan2016probabilistic,carette2016simplifying} use symbolic algebra systems to perform exact inference on all or part of a probabilistic model.
Our work can be seen as a mixed symbolic-numerical approach that provides limited symbolic pattern manipulation and relies on a high-level tensor library (PyTorch \citep{paszke2017automatic}) for automatic differentiation and parallelization.
Indeed we see functional tensors as a compromise between fully symbolic and fully numerical integration in the same way that automatic differentiation is a compromise between symbolic differentiation and numerical differentiation \citep{baydin2018automatic}.
\citet{jax2018github} describe a system of a first-order intermediate language on which machine learning program transformations are expressed as final-style interpreters \citep{carette2009finally}, with a focus on advanced automatic differentiation and hardware acceleration.

\citet{dillon2017tensorflow} describe a low-level software abstraction for implementing probability distributions, in particular taking care to implement batching and broadcasting. 
Our work can be seen as generalization of such distributions in three directions: from broadcastable dimensions to free variables, from normalized to unnormalized, and from single distributions to joint distributions (still with $O(1)$ many underlying tensors).
\citet{hoffman2018autoconj} design a system for automatic conjugacy detection in computation graphs; our system matches coarser patterns, e.g.~Gaussians rather than polynomials.

\citet{sarkka2019temporal} adapts parallel-scan algorithms to Bayesian filtering settings, demonstrating exponential parallel speedup for inference in discrete HMMs and Kalman filters.
\citet{baudart2019reactive} develop a modeling language for sequential probabilistic models together with linear-time bounded memory inference algorithms.
We generalize parallel-scan inference to a modeling language wider than \citep{sarkka2019temporal} but more restrictive than \citep{baudart2019reactive}.

\section{Experiments}

To evaluate the versatility of funsors in probabilistic programming,
we perform inference on a variety of probabilistic models
using the algorithms from Sec.~\ref{sec:algorithms}.
We include models in which
inference can be done exactly (Sec.~\ref{sec:ecohmm}-\ref{sec:bias}) as well as models for which inference is intractable but where approximate
inference algorithms can benefit from funsor computations of tractable subproblems (Sec.~\ref{sec:slds}-\ref{sec:bart}),
excercising compositions of the individual pieces of Sec.~\ref{sec:algorithms}.

\subsection{Discrete factor graphs}
\label{sec:ecohmm}

Hidden Markov models (HMMs) are widely used to analyze animal behavioral data due to their interpretable nature
and the availability of efficient exact inference 
algorithms \citep{zucchini2016hidden, mcclintock2018momentuhmm}.
Here we reproduce one such application, the model selection analysis in \citep{mcclintock2013combining}
of GPS movement data from a colony of harbor seals in the United Kingdom.

Using our parallel scan algorithms for marginal likelihood computation,
we fit four variants of a hierarchical HMM 
with no random effects (\texttt{No RE}), 
sex-level discrete random effects (\texttt{Group RE}), 
individual-level discrete random effects (\texttt{Individual RE}), 
and both types of random effects (\texttt{Individual+Group RE}).
We describe the models, dataset, and training procedure in Appendix \ref{sec:appecohmm}.

\newcommand{\aicsealgnin}{\small $299.3 \times 10^3$}
\newcommand{\aicsealgnid}{\small $288.2 \times 10^3$}
\newcommand{\aicsealgdin}{\small $288.8 \times 10^3$}
\newcommand{\aicsealgdid}{\small $288.0 \times 10^3$}
\newcommand{\tvesealgnin}{\small $3.21$}
\newcommand{\tvesealgnid}{\small $3.54$}
\newcommand{\tvesealgdin}{\small $3.70$}
\newcommand{\tvesealgdid}{\small $4.01$}
\newcommand{\fvesealgnin}{\small $0.033$}
\newcommand{\fvesealgnid}{\small $0.033$}
\newcommand{\fvesealgdin}{\small $0.042$}
\newcommand{\fvesealgdid}{\small $0.052$}
\begin{table}[t!]
\begin{center}
\resizebox {.98\columnwidth} {!} {

\begin{tabu}{|c|[1pt]c|c|c|}
    \hline
    \small Model \cellcolor[gray]{0.95} &  \small AIC & \small TVE time (s) & \small FVE time (s) \\  \tabucline[1pt]{-}

    \small \texttt{No RE} & \aicsealgnin  & \tvesealgnin & \fvesealgnin  \\ \hline
    \small \texttt{Individual RE} & \aicsealgnid & \tvesealgnid & \fvesealgnid \\ \hline
    \small \texttt{Group RE} & \aicsealgdin & \tvesealgdin & \fvesealgdin \\ \hline
    \small \texttt{Individual+Group RE} & \aicsealgdid & \tvesealgdid & \fvesealgdid \\ 
    \hline
\end{tabu}
} 
\end{center}
\caption{AIC scores and wall clock time to compute marginal likelihood and gradients of all parameters using parallel-scan funsor variable elimination (FVE) and sequential tensor variable elimination (TVE) as in \citet{obermeyer2019tensor}. See Sec.~\ref{sec:ecohmm}.
  }
\label{tbl:ecohmm}
\end{table}

\begin{figure}[t!]
  \centering
  {\includegraphics[width=0.98\columnwidth]{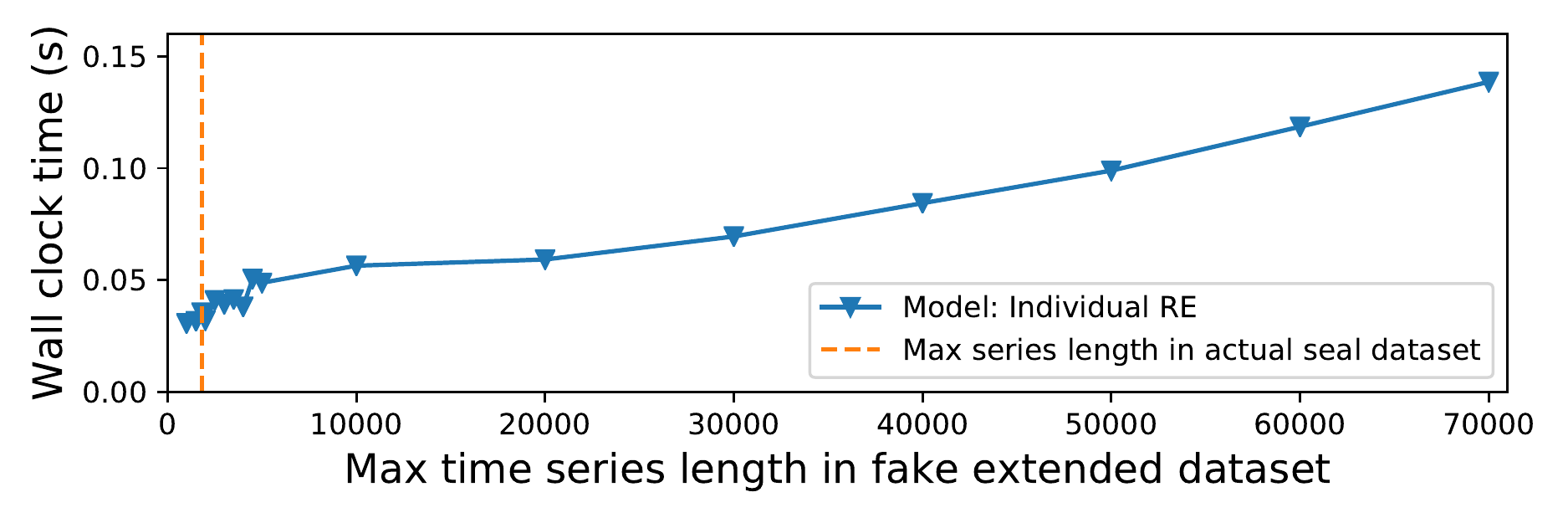}}
  \caption{Time to compute marginal likelihood and gradients of all parameters with funsor variable elimination in a simplified version of model \texttt{Individual RE} on fake seal data extended in time. See Appendix \ref{app:ecohmm-perf} for details.}
  \label{fig:hmmscaling}
\end{figure}

We report AIC scores for the four model variants in Table \ref{tbl:ecohmm}.
As in the original analysis \citep{mcclintock2013combining},
our results support the inference that there is behavioral variation
across individuals that is unexplained either by sex or the available covariates.
The times in Table \ref{tbl:ecohmm} show that our parallel-scan implementation (Sec.~\ref{sec:parallel-scan})
is more than two orders of magnitude faster than tensor variable elimination \citep{obermeyer2019tensor} for this class of models.
Figure \ref{fig:hmmscaling} shows that it achieves the expected logarithmic scaling 
for series $>10\times$ longer than the harbor seal tracks.
See Appendix \ref{app:ecohmm-perf} for details and more scaling experiments.

\subsection{Kalman filters with global latents}
\label{sec:bias}

Consider a 2-D tracking problem where an object is observed for $T$ time steps by  $S = 5$ synchronized sensors that introduce both iid noise and unknown persistent bias.
Suppose the object follows nearly-constant-velocity (linear-Gaussian) dynamics and observations, 
but that the scales of the process/observation noise and bias are unknown.

Neglecting bias, we could naively perform inference via differentiable Kalman filtering and optimize noise scales to maximize marginal likelihood.
To account for bias we add a persistent Gaussian random variable, as shown in Fig.~\ref{fig:biasfig}.
We exactly marginalize out the bias latent states, then optimize noise scales using gradient descent, leading to more accurate position estimates as in Figure~\ref{fig:biasplots}.
See Sec.~\ref{sec:appbias} for details.


\tikzstyle{factor} = [rectangle, draw]
\tikzstyle{left plate caption} = [caption, node distance=0, inner sep=0pt,
below left=0pt and 0pt of #1.south west]
\tikzstyle{right plate caption} = [caption, node distance=0, inner sep=0pt,
below left=0pt and 0pt of #1.south east]
\begin{figure}[t!]
\begin{center}
\resizebox {.45\columnwidth} {!}{
\begin{tikzpicture}
\definecolor{mygrey}{RGB}{220,220,220}

\node[latent, fill=mygrey](Y1) {$x_{t-1}$};
\node[latent, fill=mygrey, right=of Y1](Y2) {$x_{t}$};

\node[latent, above=of Y1, yshift=-1em](Z1) {$z_{t - 1}$};
\node[latent, right=of Z1](Z2) {$z_{t}$};
\node[latent, below=of Y1, xshift=2.5em, yshift=1.5em](B1) {$\beta$};

\node[const, left=of Y1](Y3) {$\cdots$};
\node[const, left=of Z1](Z3) {$\cdots$};
\node[const, right=of Y2](Y4) {$\cdots$};
\node[const, right=of Z2](Z4) {$\cdots$};

\plate [inner sep=0.9em, yshift=1.2em] {b}{(Y3)(Y4)(B1)}{};
\node [right plate caption=b-wrap, yshift=1.2em]{$S$};

\edge {Z3} {Z1}
\edge {Z1} {Z2,Y1}
\edge {Z2} {Y2, Z4}
\edge {B1} {Y1, Y2, Y3, Y4}

\end{tikzpicture}
}
\end{center}
\caption{Graphical structure of the continuous state space model in Sec.~\ref{sec:bias}.  $z$ are latent states, $x$ are observations, and $\beta$ is the persistent sensor bias.} 
\label{fig:biasfig}
\end{figure}
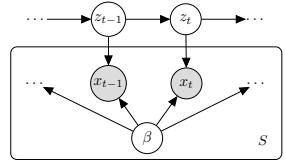

\begin{figure}[h!]  
  \centering
  {\includegraphics[width=0.98\columnwidth]{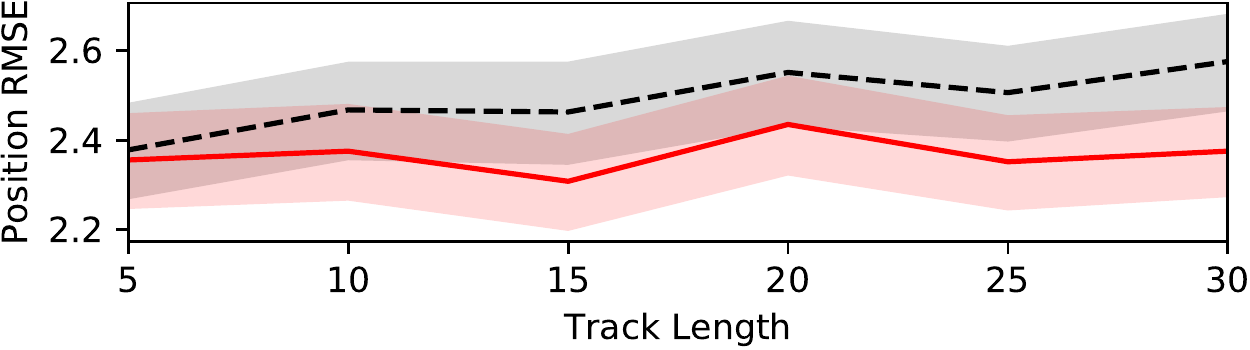}}
  \caption{Position error of the final state estimate for the autotuning Gaussian state space filter of Sec.~\ref{sec:bias} neglecting bias (dashed black) and modeling bias (solid red).}
  \label{fig:biasplots}
\end{figure}

\subsection{Switching linear dynamical system}
\label{sec:slds}

We use a switching linear dynamical system (SLDS) \citep{ackerson1970state} to model an EEG time series dataset $\{y_t\}_{t=1}^T$ 
from the UCI database
\citep{Dua:2019}.  The generative model is as follows. At each time step $t$ there is both a discrete switching label
$s_t \in [1, ..., K]$ and a continuous latent state $x_t$; both follow Markovian dynamics, see Fig.~\ref{fig:sldsdiagram}. 
We consider three model variants: I) the transition
probabilities $p(x_t | x_{t-1}, s_t)$ depend on the switching state; II) the emission probabilities $p(y_t | x_t, s_t)$ depend
on the switching state; and III) both the transition and emission probabilities depend on the switching state.  See Sec.~\ref{sec:appslds} in the supplementary materials for details.

Exact inference for this class of models is $\OO(K^T)$. To make inference tractable, we use a moment-matching approximation
with window length $L$, reducing the complexity to $\OO(K^{L+1})$. Representing this approximate inference algorithm follows immediately
by employing a \textsc{MomentMatching} interpretation for funsor reductions.\footnote{See Ex.~\ref{ex:slds} in
the appendix for example funsor code.} 
For parameter learning we use gradient ascent on the (approximate)
log marginal likelihood $\log p(y_{1:T})$.  See Table~\ref{table:slds} for the results we obtain for all three model variants with $K=2$ 
switching states and window lengths $L \in \{1, 3, 5\}$.
We obtain the best results with the richest model (SLDS-III), with the most expensive moment-matching approximation ($L=5$) yielding
the lowest mean squared error.
In Fig.~\ref{fig:slds} we depict smoothing estimates for the training data and one-step-ahead predictions 
for the held-out data using the best performing model, validating the efficacy of the moment-matching approximation. 

\begin{table}[t!]
\begin{center}
\resizebox {.85\columnwidth} {!} {
    \begin{tabu}{|c|[1pt]c|c|c|c|c|c|}    \hline
   \cellcolor[gray]{0.75} & \multicolumn{2}{c|}{\small $L = 1$ \cellcolor[gray]{0.95}}  & \multicolumn{2}{c|}{\small $L = 3$ \cellcolor[gray]{0.95}}
    & \multicolumn{2}{c|}{\small $L = 5$ \cellcolor[gray]{0.95}} \\  \hline
    \small Model \cellcolor[gray]{0.95} &  \small MSE & \small LL  &  \small MSE & \small LL &  \small MSE & \small LL \\  \tabucline[1pt]{-}
    \small SLDS-I & \small 0.574 & \small -10.13 & \small 0.574 & \small -10.13 & \small 0.574 & \small -10.13  \\ \hline
     \small SLDS-II & \small 0.527 & \small -9.55 & \small 0.497 & \small -9.64 & \small 0.498 & \small -9.64  \\ \hline
    \small SLDS-III & \small 0.512 & \small {\bf -9.33} & \small 0.511 & \small -9.41 & \small {\bf 0.482} & \small -9.46  \\
    \hline
    \end{tabu}
} 
\end{center}
     \caption{One-step-ahead test log likelihoods and mean squared errors for SLDS variants with various moment-matching window lengths $L$. 
     See Sec.~\ref{sec:slds} for details.}
\label{table:slds}
\end{table}

\begin{figure}[t]
  \centering
  {\includegraphics[width=\columnwidth]{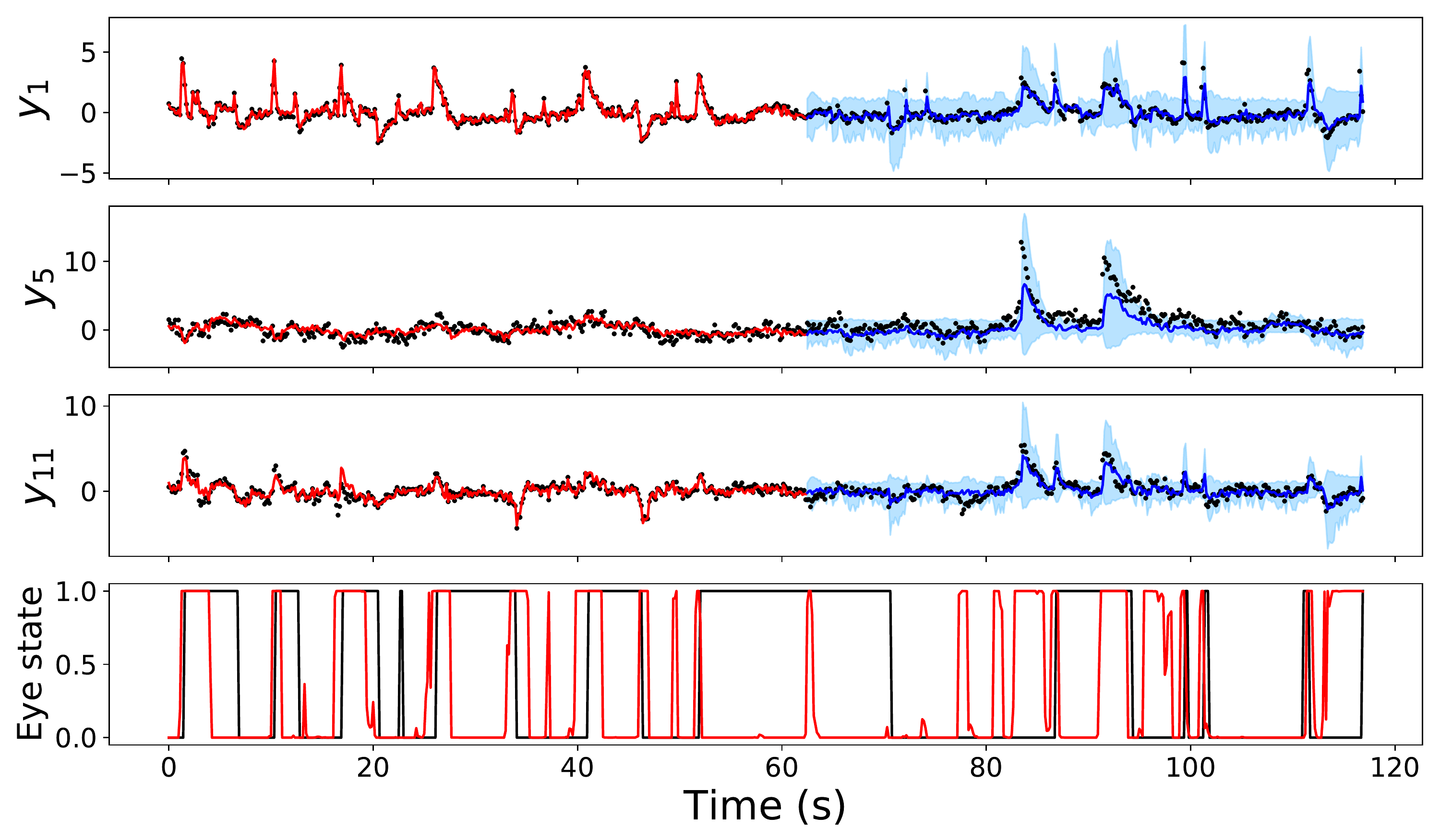}}
  \caption{Smoothing estimates (red) and one-step-ahead predictions (blue; with 90\% confidence intervals) for three 
  randomly selected output dimensions $(y_1, y_5, y_{11})$ for the SLDS experiment in \ref{sec:slds}. The bottom figure depicts
  the observed eye state (black) as well as the smoothing estimate of the inferred switching label $s_t$ (red). Note that the eye state was unobserved during training.}
  \label{fig:slds}
\end{figure}


\tikzstyle{factor} = [rectangle, draw]
\tikzstyle{left plate caption} = [caption, node distance=0, inner sep=0pt,
below left=0pt and 0pt of #1.south west]
\tikzstyle{right plate caption} = [caption, node distance=0, inner sep=0pt,
below left=0pt and 0pt of #1.south east]

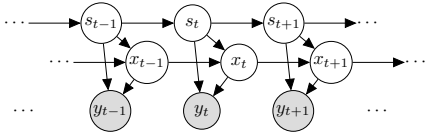
\begin{figure}[t!]
\begin{center}
\resizebox {.7\columnwidth} {!} {
\begin{tikzpicture}
\definecolor{mygrey}{RGB}{220,220,220}

\node[latent, fill=mygrey](Y1) {$y_{t-1}$};
\node[latent, fill=mygrey, right=of Y1](Y2) {$y_{t}$};
\node[latent, fill=mygrey, right=of Y2](Y3) {$y_{t+1}$};
\node[latent, above=of Y1,xshift=2em,,yshift=-2.5em](X1) {$x_{t - 1}$};
\node[latent, above=of X1,xshift=-2.5em,yshift=-3em](S1) {$s_{t - 1}$};
\node[latent, right=of X1](X2) {$x_{t}$};
\node[latent, right=of X2](X3) {$x_{t+1}$};
\node[latent, right=of S1](S2) {$s_{t}$};
\node[latent, right=of S2](S3) {$s_{t+1}$};

\node[const, left=of Y1](Y0) {$\cdots$};
\node[const, left=of X1](X0) {$\cdots$};
\node[const, right=of X3](X4) {$\cdots$};
\node[const, left=of S1](S0) {$\cdots$};
\node[const, right=of S3](S4) {$\cdots$};
\node[const, right=of Y3](Y4) {$\cdots$};

\edge {S1} {X1,S2,Y1}
\edge {S2} {X2,Y2,S3}
\edge {X0} {X1}
\edge {S0} {S1}
\edge {S3} {X3,S4,Y3}
\edge {X1} {X2,Y1}
\edge {X2} {Y2,X3}
\edge {X3} {Y3,X4}

\end{tikzpicture}
}
\end{center}
\caption{Graphical structure for the SLDS-III model in Sec.~\ref{sec:slds}. The $\{s_t\}$ form
a chain of discrete switching states and the $\{x_t\}$ are continuous. Here the transition
probabilities $p(x_t | x_{t-1}, s_t)$ and emission probabilities $p(y_t | x_t, s_t)$ both depend on the switching label $s_t$.}
\label{fig:sldsdiagram}
\end{figure}

\subsection{Neural variational Kalman filter}
\label{sec:bart}



We model a high-dimensional count-valued time series by combining exact computations with variational inference.
The data we consider are hourly ridership counts for every pair of 47 stations in a metropolitan transit system, 
totalling over 78M nonzero observations \citep{bart2019ridership}.
Our generative model consists of a low-dimensional (dim $\in\{2,4,8\}$) linear state space model with high-dimensional (one dimension per origin-destination pair) zero-inflated Poisson observations that depend non-linearly (via a neural network) on the low-dimensional latent state, the hour of week, and a boolean vector encoding whether each station is open or closed (see Fig.~\ref{fig:bart}).

We consider two inference strategies combining amortized variational inference with exact marginalization.
The first ``mean field'' strategy inputs a block of observed counts $c_{1:T}$ and predicts fully independent normal distributions over latent gate probability $p_{1:T}$, Poisson rate $\lambda_{1:T}$, and state $z_{1:T}$; the gate variable $g_{1:T}$ is marginalized out.
The second ``collapsed'' strategy inputs a single timestep of observed counts $c$, independently predicts fully independent normal distributions over $p$ and $\lambda$, and exactly marginalizes $z$ using the parallel-scan Markov product Alg.~\ref{alg:markov-product}.
In both strategies we jointly train the generative model and inference model using stochastic variational inference on 2-week long mini-batches of data.
We then condition on the final 2 weeks and predict forward 1 week.
Table~\ref{table:bart} shows forecast accuracy.
See Appendix~\ref{sec:appbart} for details.


\tikzstyle{left plate caption} = [caption, node distance=0, inner sep=0pt, below left=0pt and 0pt of #1.south west]
\tikzstyle{right plate caption} = [caption, node distance=0, inner sep=0pt, below left=0pt and 0pt of #1.south east]
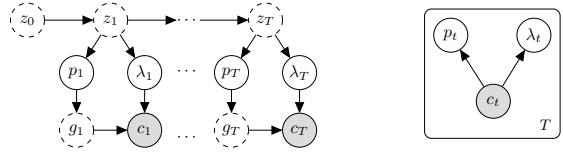
\begin{figure}[t!]
\begin{center}
\resizebox {0.9\columnwidth} {!} {
\begin{tikzpicture}
  \definecolor{mygrey}{RGB}{220,220,220}

  \node[latent, dashed](Z0) {$z_0$};

  \node[latent, dashed, right=of Z0](Z1) {$z_1$};
  \node[latent, below=of Z1, xshift=-2em, yshift=1.8em](P1) {$p_1$};
  \node[latent, below=of Z1, xshift=2em, yshift=1.8em](R1) {$\lambda_1$};
  \node[latent, dashed, below=of P1, yshift=1.5em](G1) {$g_1$};
  \node[latent, fill=mygrey, below=of R1, yshift=1.5em](C1) {$c_1$};

  \node[const, right=of Z1](Zdots) {$\cdots$};
  \node[const, below=of Zdots, yshift=0.3em](Pdots) {$\cdots$};
  \node[const, below=of Pdots, yshift=-0.5em](Cdots) {$\cdots$};

  \node[latent, dashed, right=of Zdots](ZT) {$z_T$};
  \node[latent, below=of ZT, xshift=-2em, yshift=1.8em](PT) {$p_T$};
  \node[latent, below=of ZT, xshift=2em, yshift=1.8em](RT) {$\lambda_T$};
  \node[latent, dashed, below=of PT, yshift=1.5em](GT) {$g_T$};
  \node[latent, fill=mygrey, below=of RT, yshift=1.5em](CT) {$c_T$};

  \node[latent, right=of RT, xshift=4em, yshift=2.2em](Pt) {$p_t$};
  \node[latent, right=of Pt](Rt) {$\lambda_t$};
  \node[latent, fill=mygrey, below=of Pt, xshift=2.5em, yshift=1em](Ct) {$c_t$};


  \plate [inner sep=0.5em] {b}{(Pt)(Rt)(Ct)}{};
  \node [right plate caption=b-wrap]{$T$};

  \edge {Z0} {Z1}
  \edge {Z1} {P1,R1,Zdots}
  \edge {Zdots} {ZT}
  \edge {P1} {G1}
  \edge {R1} {C1}
  \edge {G1} {C1}
  \edge {ZT} {PT,RT}
  \edge {PT} {GT}
  \edge {RT} {CT}
  \edge {GT} {CT}
  \edge {Ct} {Pt}
  \edge {Ct} {Rt}
\end{tikzpicture}
}
\end{center}
  \caption{Graphical structure of the generative model (left) and collapsed variational inference model (right) for ridership forecasting.
    Grey nodes are observed, solid white nodes are sampled via stochastic variational inference, and dashed white nodes are exactly marginalized out.
  }
\label{fig:bart}
\end{figure}

\begin{table}[t!]
  \begin{center}
  \resizebox {.8\columnwidth} {!} {
    \begin{tabu}{|c|[1pt]c|c|c|c|c|c|}
      \hline
      Inference \cellcolor[gray]{0.95}
      & \multicolumn{3}{c|}{Mean Field}
      & \multicolumn{3}{c|}{Collapsed-MC}
      \\ \hline
      Latent dim \cellcolor[gray]{0.95}
          & 2 & 4 & 8
          & 2 & 4 & 8
      \\ \hline
      MAE \cellcolor[gray]{0.95}
          & 2.39 & 2.35 & \textbf{2.33}
          & 2.65 & 2.59 & 2.76
      \\ \hline
      CRPS \cellcolor[gray]{0.95}
          & 1.78 & 1.76 & \textbf{1.74}
          & 1.95 & 1.93 & 2.06
      \\ \hline
    \end{tabu}
  } 
  \end{center}
  \caption{
    Mean Absolute Error and mean Continouous Ranked Probability Score \citep{gneiting2007strictly} of 1-week forecasts based on two inference strategies for a state space model with varying latent state dimension.
  }
  \label{table:bart}
\end{table}

\section{Conclusion}

We introduced \emph{funsors}, a software abstraction that generalizes tensors to provide finite representations for a restricted class of discrete and continuous distributions, including lazy algebraic expressions, non-normalized Gaussian distributions, and Dirac delta distributions.
We demonstrated how funsors can be integrated into a probabilistic programming system, enabling a wide variety of inference strategies. 
Finally, we have implemented funsors and funsor algorithms as a Python library built on PyTorch, with source code available at \anonymize{\url{https://github.com/pyro-ppl/funsor}}.

In future work, we could follow AD systems in increasing our term language's expressivity by adding simple control flow constructs like loops, branching or fixpoints, and with basic union/pair types to represent custom data structures.
We could also expand the class of inference algorithms represented in our formalism by adding interpreters that approximate other operations.
For example, just as our \textsc{MomentMatching} interpreter approximates the ``sum'' operation, 
we could naturally represent other algorithms like mini-bucket elimination \citep{dechter1997mini} or one step of loopy belief propagation
by also approximating the binary ``product'' operation that constructs full joint distributions.





\bibliography{enumeration}

\newpage
\appendix
\section{A Markov product operation} \label{sec:markov-product}

Plates and Markov chains are ubiquitous motifs in structured probabilistic modeling.
We define a basic operation that unifies pointwise products over plates, chained matrix multiplication, and Bayesian filtering.
The Markov product generalize the usual plated product, with syntax given by
\begin{align*}
  e \in \op{Funsor} ::= \dots 
    \Mid & \Prod_{v/s} e  && \text{``Markov product''}
\end{align*}
and semantics given by the following definition:
\begin{definition} \label{def:markov-product}
  Let $f$ be a funsor with output type $\tau=\mathbb R$,
  let $t:\mathbb Z_T$ be a free variable over $T\ge1$ ``time steps'',
  and let $s\subseteq\op{fv}(f)\times\op{fv}(f)$ be a partial ``time step'' matching among the free variables of $f$ such that:
  (i) $t$ does not appear in $s$;
  (ii) $s$ is one-to-one, i.e.~${\left|\{u\mid (u,v)\in s\}\cup\{v\mid(u,v)\in s\}\right|}=2|s|$; and
  (ii) for every pair $(u,v)\in s$, $u$ and $v$ are identically typed in context $\Gamma$.
  We define the \emph{Markov product} $\prod_{t/s} f$ of $f$ along variable $t$ modulo $s$ by induction on $T$:
  $$
    \Prod_{t/s} f = \begin{cases}
      f[t\subs 0] \hspace{2cm}\text{if } T = 1 \text{; otherwise }\\
    \sum\limits_{\underline w}
      f[t \subs T, \underline u \subs \underline w]
      \times\prod\limits_{t/s} f[t:\mathbb Z_{T-1}, \underline v \subs \underline w]
    \end{cases}
  $$
  where $\underline w$ is a tuple of $|s|$-many fresh variables,
  $\underline u=\op{dom} s={\left(u\mid (u,v)\in s\right)}$ is the domain tuple of $s$,
  $\underline v=\op{cod} s={\left(v\mid (u,v)\in s\right)}$ is the codomain tuple of $s$,
  ${f[t:\mathbb Z_{T-1}]}$ is the prefix of $f$ to the first $T-1$ time steps, and
  $\sum_{w}$ denotes either summing out a discrete variable or integrating out a real-tensor variable.
\end{definition}
Note the time step mapping $s$ corresponds to the \texttt{pre} operator in \citep{baudart2019reactive}, and formalizes the idiom of marking variable names with time lags like \texttt{x\_prev}, \texttt{x\_curr}.

\begin{example}
In the simplest case of empty matching $s$, the Markov product reduces to the usual product
$$
  \Prod_{v/\emptyset} f
  \;=\; \Prod_v f
  \;=\; f[t\subs 0] \times \cdots \times f[t\subs T-1].
$$
\end{example}

\begin{example}
  Let $f$ be a funsor with shape ${t\oftype\mathbb Z_T, i\oftype\mathbb Z_N, j\oftype\mathbb Z_N \vdash f:\mathbb R}$, equivalent to a batch of $N\times N$ matrices.
  Let $s=\{(i,j)\}$ map a single previous discrete state $i$ to a current state $j$.
  Then the Markov product is equivalent to a chain of matrix multiplies
  $$
    \Prod_{t/(i,j)} f \;=\; f[t \subs 0] \bullet f[t \subs 1] \bullet \cdots \bullet f[t \subs T-1],
  $$
  where each binary operation can be defined as a sum-product expression,
  which is the core computation in variable elimination in discrete Markov models,
  $$
    f \bullet g = \Sum_k f[j \subs k] \times g[i \subs k].
  $$
\end{example}

\begin{example}
  Let $f$ be a funsor with shape ${t\oftype\mathbb Z_T, x_{\text{prev}}\oftype\mathbb R^3, x_{\text{curr}}\oftype\mathbb R^3 \vdash f:\mathbb R}$ defined by a density with two conditional multivariate normal factors
  $$
  f = \mathcal{MVN}(x_{\text{curr}}; F x_{\text{prev}}, P) \times \mathcal{MVN}(y; H x_{\text{curr}}, Q)
  $$
  corresponding to a dynamical system with linear dynamics $F\in\mathbb R^{3\times 3}$, process noise covariance $P$, linear observation matrix $H\in\mathbb R^{2\times 3}$, observation noise covariance $Q$, and observations $y$ that depend on time, $t\oftype\mathbb Z_T\vdash y\oftype\mathbb R^2$.
  This Markov product is equivalent to a Kalman filter, producing a joint distribution over the initial and final states.
  The final state distribution is given by further marginalizing out the initial state:
  \begin{align*}
    x_{\text{prev}}\oftype\mathbb R^3,\, x_{\text{curr}}\oftype\mathbb R^3 \;
      &\vdash \; \phantom{\Sum_{x_{\text{prev}}} \; } \Prod_{t/(x_{\text{prev}},x_{\text{curr}})} f : \mathbb R \\
    x_{\text{curr}}\oftype\mathbb R^3 \;
      &\vdash \; \Sum_{x_{\text{prev}}} \; \Prod_{t/(x_{\text{prev}},x_{\text{curr}})} f : \mathbb R
  \end{align*}
\end{example}

The Markov product operation generalizes the Bayesian filtering operations of \citep{sarkka2019temporal} to multiple latent random variables;
Sec.~\ref{sec:parallel-scan} extends their temporal parallelization algorithms to Markov products of funsors.

\subsection{Parallel-scan Bayesian filtering} \label{sec:parallel-scan-algo}

We implement parallel-scan Bayesian filtering \citep{sarkka2019temporal}  as a parallel-scan rewrite rule Algorithm~\ref{alg:markov-product} for the Markov product operation.

\begin{algorithm}[t!]
  \caption{\textsc{MarkovProduct}}
  \label{alg:markov-product}
  {\bf input} a funsor $f$, a time variable $t\in\op{fv}(f)$, \\
  \tab a step mapping $s\subseteq\op{fv}(f)\times\op{fv}(f)$. \\
  {\bf output} the Markov product funsor $\prod_{t/s} f$.\\
  Create substitutions with fresh names (barred): \\
  $s_e\from \left\{(y, \bar x) \mid (x, y)\in s\right\}$ to rename even factors, and \\
  $s_o\from \left\{(x, \bar x) \mid (x, y)\in s\right\}$ to rename odd factors.\\
  Let $v \from \left\{\bar x \mid (x, y) \in s\right\}$ be variables to marginalize. \\
  Let $T \from \left|\Gamma_f[t]\right|$ be the length of the time axis. \\
  \While{$T > 1$}{
    Split $f$ into even and odd parts of equal length: \\
    $f_e \from f[s_e,\, t\subs(0,2,4,6,..., 2\lfloor T/2 \rfloor-2)]$ \\
    $f_o \from f[s_o,\, t\subs(1,3,5,7,..., 2\lfloor T/2 \rfloor-1)]$ \\
    Perform parallel sum-product contraction: \\
    $f' \from \sum_v f_e \times f_o$ \\
    \leIf{$T$ is even}{
      $f\from f'$\\
    }{
      $f\from \op{concat}_t\left(f', f[t\subs T-1]\right)$
    }
    $T \from \lceil T/2 \rceil$
  }
  \Return $f[t\subs 0]$
\end{algorithm}

\begin{theorem}
  \textsc{MarkovProduct} Algorithm~\ref{alg:markov-product} has parallel complexity logarithmic in time length $T$.
\end{theorem}
\begin{proof}
  Each funsor operation parallelizes over time, and the \textbf{while} loop executes $O(\log(T))$ many times, hence total parallel complexity is $O(\log(T))$.
\end{proof}

\section{Syntax and operational semantics} \label{sec:rules}

\subsection{Type inference} \label{sec:rules.types}

Type inference rules for funsors are presented in Fig.~\ref{fig:typing-rules}.
Most type inference rules for funsors generalize shape inference rules for tensors.
We annotate such rules with their tensor analogs in NumPy,
e.g. the $\op{Delta}$ funsor generalizes one-hot encoded arrays;
the $\op{Variable}$ funsor generalizes reshaped identity matrices;
substitution $e1[v\subs e_2]$ corresponds to indexing via brackets or \texttt{np.take};
and function lifting $\widehat{f}(\cdots)$ corresponds to broadcasting.

\newcommand \torch [1] {{\hspace{1em}``\texttt{\small #1}"}}

\begin{figure}[ht!]
  \begin{mathpar}

    \inferrule {
      \Gamma = (v_1 \oftype \mathbb Z_{d_1}, \dots, v_n \oftype \mathbb Z_{d_n}) \\\\
      w \in \mathbb Z_{d_1} {\times} \cdots {\times} \mathbb Z_{d_n} \to \tau
    }{
      \Gamma \vdash \op{Tensor}(\Gamma,\tau,w) : \tau
    }{\torch{np.array}}

    \inferrule {
      s = s_1 {\times} \cdots {\times} s_n \\\\
      i \in \mathbb Z_{d_1} {\times} \cdots {\times} \mathbb Z_{d_1} \to \mathbb R^s \\
      \precision \in \mathbb Z_{d_1} {\times} \cdots {\times} \mathbb Z_{d_1} \to \mathbb R^s\times\mathbb R^s \text{ p.s.d.} \\
      \Gamma = (u_1 \oftype \mathbb Z_{d_1}, \dots, u_m \oftype \mathbb Z_{d_m},
                v_1 \oftype \mathbb R^{s_1}, \dots, v_n \oftype \mathbb R^{s_n})
    }{
      \Gamma \vdash \op{Gaussian}(\Gamma,i,\precision) : \mathbb R
    }

    \inferrule {
      \Gamma \vdash x \oftype \tau
    }{
      \Gamma, v\oftype\tau \vdash \op{Delta}(v, x) : \mathbb R
    }{\torch{one\_hot}}

    \inferrule {
    }{
      v\oftype\tau \vdash \op{Variable}(v, \tau) : \tau
    }{\torch{np.eye}}

    \inferrule{
      \Gamma, v \oftype \tau_2 \vdash e_1 \oftype \tau_1 \\
      \Gamma \vdash e_2 \oftype \tau_2
    }{
      \Gamma \vdash e_1[v \subs e_2] : \tau_1
    }{\torch{np.take}}

    \inferrule{
      \Gamma, v \oftype \tau \vdash e \oftype \mathbb R^s
    }{
      \Gamma \vdash \Sum_v e : \mathbb R^s
    }{\torch{np.sum}}

    \inferrule{
      \Gamma, v \oftype \mathbb Z_n \vdash e \oftype \mathbb R^s
    }{
      \Gamma \vdash \Prod_v e : \mathbb R^s
    }{\torch{np.prod}}

    \inferrule{
      \Gamma, v \oftype \mathbb Z_n \vdash e \oftype \mathbb R^s \\
      \sigma \subseteq \op{fv}(e) \times \op{fv}(e) \text{ valid}
    }{
      \Gamma \vdash \Prod_{v/\sigma} e : \mathbb R^s
    }

    \inferrule{
      f \in \tau_1\times\dots\times\tau_n \to \tau_0 \\\\
      \Gamma\vdash e_1 \oftype \tau_1 \\
      \cdots \\
      \Gamma\vdash e_n \oftype \tau_n
    }{
      \Gamma \vdash \lift f(e_1, \dots, e_n) : \tau_0
    }{\torch{broadcast}}

  \end{mathpar}
  \caption{
    Typing rules for funsors.
    We use set notation $\in$ to denote numerical objects like functions $f\in \tau_1 \to \tau_2$ and multidimensional arrays $x\in \mathbb Z_3 \times \mathbb Z_3 \to \mathbb R$.
    We reserve type notation $e:\tau$ for funsors $e$.
    The precision matrix $\precision$ in a Gaussian must be symmetric positive semidefinite for all values of batch variables $(u_1,\dots,u_m)$.
    The step substitution $\sigma$ in the Markov product must be a valid matching as defined Definition~\ref{def:markov-product}.
  }
  \label{fig:typing-rules}
\end{figure}

Because funsor types include shape information, we can perform shape checking even under a lazy interpretation.
While this is less powerful than static shape checking (which would require static analysis of Python code), it does allow us to catch errors before expression optimization, evaluation, or approximation.
%
Most usefully, funsor types prevent a large class of \emph{broadcasting bugs}, which have proved to be a common class of bugs, especially when using discrete variable elimination algorithms for inference.

\subsection{Term rewriting} \label{sec:rules.rewriting}

Rewrite rules are specified by registering a pattern together with a handler to execute when the pattern is matched.
Figs.~\ref{fig:rewrite-rules-lazy}-\ref{fig:rewrite-rules-approx} present a subset of the rewriting rules defining operational semantics in different interpretations.
Fig.~\ref{fig:example-rules} provides two example rules.
Some properties of rules are too complex for the simple pattern matching mechanism, e.g. the conditions on the right of Fig.~\ref{fig:rewrite-rules-lazy}.
We support this finer grained matching by allowing extra matching logic in the handler for each pattern, whereby the handler can reject a pattern match.

\begin{figure}[h!]
\begin{lstlisting}[language=Python]
@eager.register(Binary, Op, Tensor, Tensor)
def eager_binary_tensor(op, lhs, rhs):
    assert lhs.dtype == rhs.dtype
    inputs, (x, y) = align_tensors(lhs, rhs)
    data = op(x, y)
    return Tensor(data, inputs, lhs.dtype)

@eager.register(Binary, AddOp, Delta, Funsor)
def eager_add_delta(op, lhs, rhs):
    if lhs.name in rhs.inputs:
        rhs = rhs(**{lhs.name: lhs.point})
        return op(lhs, rhs)
    return None  # defer to default implementation
\end{lstlisting}
\caption{
  Example rewriting rules for \lstinline$Binary$ addition, a special case of lifted function $\widehat f(\text-,\text-)$.
  The \lstinline$eager_binary_tensor$ rule performs a tensor operation.
  The \lstinline$eager_add_delta$ rule performs extra matching logic.
} \label{fig:example-rules}
\end{figure}

\begin{figure*}[p!]
  \begin{align*}
    \text{\parbox{12em}{\fbox{\textsc{Lazy}}}} & \\
    \op{Delta}(v_1, e_1)[v_2 \subs e_2]
      & \redto \op{Delta}(v_1, e_1[v_2 \subs e_2])
       && \text{if $v_1 \ne v_2$ and $v_1 \notin \op{fv}(e_2)$} \\
    \\[-1em]
    \op{Variable}((v\oftype\tau))[v \subs e] & \redto e \\
    \op{Variable}((v\oftype\tau))[v' \subs e] & \redto \op{Variable}((v\oftype\tau)) && \text{if $v \neq v'$}\\
    \\[-0.5em]
    e_1[v\subs e_2] & \redto e_1 && v\notin \op{fv}(e_1) \\
    \widehat f\left(e_1, \dots, e_n\right)[v \subs e_0]
      & \redto \widehat f\left(e_1[v \subs e_0], \dots, e_n[v \subs e_0]\right) \\
    \bigl(\Sum_{v_1} e_1\bigr)[v_2 \subs e_2]
      & \redto \Sum_{v_1} e_1[v_2 \subs e_2] 
       && v_1 \notin \op{fv}(e_2) \text{ and } v_1,v_2 \text{ distinct}\\
    \bigl(\Prod_{v_1/s} e_1\bigr)[v_2 \subs e_2]
      & \redto \Prod_{v_1/s} e_1[v_2 \subs e_2] 
       && v_1 \notin \op{fv}(e_2) \text{ and } v_1,s,v_2 \text{ all distinct}\\
  \end{align*}
  \caption {
    Selected rewrite rules for the \textsc{Lazy} interpretation.
    This interpretation only handles substitution into non-atomic terms,
    and unlike \textsc{Exact} none of its rules trigger numerical computation.
    Terms that do not match any of the rules are reflected, i.e. evaluated to lazy versions of themslves.
    Additional rules include normalization rules with respect to associativity, commutativity, and distributivity.
    Bound variables are $\alpha$-renamed to avoid conflict during substitution.
  }
  \label{fig:rewrite-rules-lazy}
\end{figure*}

\begin{figure*}[p!]
  \begin{align*}
    \text{\parbox{12em}{\fbox{\textsc{Exact}}}} & \\
    \op{Delta}(v, e_1) \times e_2
      & \redto \op{Delta}(v, e_1) \times e_2[v \subs e_1]
       && \text{if $v \in \op{fv}(e_2)$} \\
    \Sum_v \op{Delta}(v, e) & \redto 1 \\
    \Sum_v \op{Delta}(v, e_1) \times e_2 
      & \redto e_2 
        && \text{if $v \notin \op{fv}(e_2)$} \\
    \\[-1em]
    \op{Variable}( (v\oftype \mathbb Z_n) )
      & \redto \op{Tensor}( ( v\oftype \mathbb Z_n ), \texttt{arange}(n) ) \\
    \\[-1em]
    \op{Tensor}(\Gamma_1, w_1)[v \subs \op{Tensor}(\Gamma_2, w_2)] 
      & \redto \op{Tensor}( ( \Gamma_1 - (v\oftype\tau) ) \cup \Gamma_2, \texttt{index}(w_1, v, w_2) ) 
      && \text{if $(v\oftype\tau) \in \Gamma_1$} \\
    \widehat f\left(\op{Tensor}(\Gamma_1, w_1), \dots, \op{Tensor}(\Gamma_n, w_n)\right)
      & \redto \op{Tensor}(\cup_k \Gamma_k, f\left(w_1, \dots, w_n)\right) \\
    \Sum_v \op{Tensor(\Gamma, w)}
      & \redto \op{Tensor(\Gamma - (v\oftype\tau), \texttt{sum}(w, v))} 
      && \text{if $(v\oftype\tau) \in \Gamma$} \\
    \Prod_v \op{Tensor(\Gamma, w)}
      & \redto \op{Tensor(\Gamma - (v\oftype\tau), \texttt{prod}(w, v))} 
      && \text{if $(v\oftype\tau) \in \Gamma$} \\
    \\[-1em]
    \op{Gaussian}(\Gamma_1, i, \precision)[v \subs \op{Tensor}(\Gamma_2, w)]
      & \redto \op{Gaussian}( (\Gamma_1 - (v\oftype\tau)) \cup \Gamma_2, i^\prime, \precision^\prime) 
      && \text{if a $v_2 \neq v \in \Gamma_1$ is real-valued} \\
    \op{Gaussian}(\Gamma_1, i, \precision)[v \subs \op{Tensor}(\Gamma_2, w)]
    & \redto \op{Tensor}( (\Gamma_1 - (v\oftype\tau)) \cup \Gamma_2, w^\prime) 
      && \text{if only $v \in \Gamma_1$ is real-valued} \\
    \op{Gaussian}(\Gamma_1, i_1, \precision_1) \times \op{Gaussian}(\Gamma_2, i_2, \precision_2)
      & \redto \op{Gaussian}(\Gamma_1 \cup \Gamma_2, i_1 + i_2, \precision_1 + \precision_2) \\
    \Prod_{v} \op{Gaussian}(\Gamma, i, \precision)
    & \redto \op{Gaussian}(\Gamma - (v\oftype\mathbb Z_n), \texttt{sum}(i, v), \texttt{sum}(\precision, v) )
       && \text{if $v$ is a bounded integer variable} \\
    \Sum_v \op{Gaussian}(\Gamma, i, \precision) 
      & \redto \op{Tensor}(\Gamma_d, w) \times \op{Gaussian}(\Gamma - (v\oftype\tau), i^\prime, \precision^\prime) 
        && \text{if $v \in \Gamma$ is a real array variable} \\
      \Sum_v \op{Tensor}(\Gamma_1, w) \times \op{Gaussian}(\Gamma_2, i, \precision)
      & \redto \op{Tensor}(\Gamma_1, w) \times \Sum_v \op{Gaussian}(\Gamma_2, i, \precision)
        && \text{if $v \in \Gamma_2$ is a real array variable} \\
    \\
  \end{align*}
  \caption {
    Selected rewrite rules for the \textsc{Exact} interpretation.
    Terms that do not match any of these rules are evaluated with \textsc{Lazy}.
    Some rules trigger a numerical computation, as described in the main text and indicated with \texttt{monospace font}.
    The Gaussian rules for substitution and marginalization perform standard multivariate Gaussian computations that can be found in most graduate statistics or machine learning textbook, and as such we do not specify them in detail in this paper.
    Additional rules assist with pattern matching, including normalization rules with respect to associativity, commutativity, and distributivity.
    Bound variables are $\alpha$-renamed to avoid conflict during substitution.
  }
  \label{fig:rewrite-rules-exact}
\end{figure*}


\begin{figure*}[p!]
  \begin{align*}
    \text{\parbox{12em}{\fbox{\textsc{Exact}}}} & \\
    \Prod_{v/c} e
      & \redto \textsc{MarkovProduct}(e, v, c)
        && \text{if $v$ is a bounded integer variable} \\
        & && \text{and $c \subseteq (\op{fv}(e) - v) \times (\op{fv}(e) - v)$ } \\
    \\
    \text{\parbox{12em}{\fbox{\textsc{Optimize}}}} & \\
    \Sum_V e & \redto e && \text{if $V \subseteq \op{fv}(e)$ is empty} \\
    \\
    \Sum_V e_1 \times \cdots \times e_n
      & \redto \Sum_{V'} (\Sum_{V_1 \cap V} e_1) \times \cdots \times (\Sum_{V_n \cap V} e_n )
      && \text{where $V_k = \op{fv}(e_k) - \mathop{\cup}\limits_{j \neq k} \op{fv}(e_j)$ } \\
      & && \text{and where $V' = V - \mathop{\cup}\limits_{k} V_k $ } \\
      & && \text{if $V \subseteq \op{fv}(e_1 \times \cdots \times e_n)$ } \\
      & && \text{and if any $v \in V$ appear in only one $e_k$} \\
    \\
    \Sum_V e_1 \times \cdots \times e_n
      & \redto \Sum_{V - V'} (\Sum_{V' \cap V} e_{\sigma_1} \times e_{\sigma_2} ) \times e_{\sigma_3} \times \cdots \times e_{\sigma_n} 
      && \text{where $\sigma_1, \ldots, \sigma_n$ is a min-cost path} \\
      & && \text{where $V' = \op{fv}(e_{\sigma_1} \times e_{\sigma_2}) - \mathop{\cup}\limits_{k > 2} \op{fv}(e_{\sigma_k})$ } \\
      & && \text{if $V \subseteq \op{fv}(e_1 \times \cdots \times e_n)$ } \\
      & && \text{and if all $v \in V$ appear in $\geq 2$ $e_k$} \\
  \end{align*}
  \caption {
    Selected rewrite rules for the variable elimination algorithms of Sec.~\ref{sec:algorithms}.
    \textsc{MarkovProduct} means invoking Algorithm \ref{alg:markov-product},
    described in Sec.~\ref{sec:parallel-scan} of the main text and Sec.~\ref{sec:markov-product} of the appendix,
    all of whose steps are also evaluated with the \textsc{Exact} interpretation.
    Terms that do not match rules in the \textsc{Optimize} interpretation are evaluated with the \textsc{Lazy} interpretation.
    The permutation $\sigma_1, \ldots, \sigma_n$ in the final rule of \textsc{Optimize} has the following form: 
    $\sigma_1, \sigma_2$ are selected via $\argmin_{k,k'} \text{cost}(\op{fv}(e_k \times e_{k'}) - \mathop{\cup}\limits_{j \neq k,k'} \op{fv}(e_j), e_k, e_{k'})$ where $\text{cost}(V', e_k, e_{k'})$ is a heuristic cost function, such as memory usage, representing the computational cost of a binary sum-product operation,
    and the remaining $\sigma_3, \ldots, \sigma_n$ follow the original term ordering.
    This overall approach to variable elimination or tensor contraction is not novel.
    Indeed, our implementation directly uses the tensor contraction library \texttt{opt\_einsum} \citep{smith2018opt_einsum} to compute the $\sigma_k$s; for this reason we defer to their paper and documentation for detailed formal explanations.
  }
  \label{fig:rewrite-rules-ve}
\end{figure*}

\begin{figure*}[p!]
  \begin{align*}
    \text{\parbox{12em}{\fbox{\textsc{MomentMatching}}}} & \\
    \Sum_v \op{Tensor}(\Gamma_1, w) \times \op{Gaussian}(\Gamma_2, i, \precision)
      & \redto \op{Gaussian}(\Gamma_2 - (v\oftype\tau), i^\prime, \precision^\prime) \times \Sum_v \op{Tensor}(\Gamma_1, w^\prime)
       && \text{if $v$ is a bounded integer variable} \\
    \\
  \end{align*}

  \begin{align*}
    \text{\parbox{12em}{\fbox{\textsc{MonteCarlo}}}} & \\
    \Sum_v \op{Tensor}(\Gamma, w) \times e
      & \redto \op{Tensor}(\Gamma - (v\oftype\tau), w_N \times w_D) \times \Sum_v \op{Delta(v, e_s)} \times e 
      && \text{if $v$ is a bounded integer variable} \\
      \Sum_v \op{Gaussian}(\Gamma_d \cup (v\oftype\tau), i, \precision) \times e
      & \redto \op{Tensor}(\Gamma_d, w_N) \times \Sum_v \op{Delta(v, e_s)} \times e
      && \text{if $v$ is a real array variable} \\
  \end{align*}
  \caption {
      Selected rewrite rules for the approximate \textsc{MonteCarlo} and \textsc{MomentMatching} interpretations of Sec.~\ref{sec:algorithms}.
    Terms that do not match any rules in the \textsc{MonteCarlo} and \textsc{MomentMatching} interpretations are evaluated with the \textsc{Exact} interpretation.
    In the moment matching rules, the new values $i^\prime,\precision^\prime,w^\prime$ are computed to match the moments of the original expression as discussed in Sec.~\ref{sec:moment-matching} of the main text.
    In the \textsc{MonteCarlo} rules, $e_s$ refers to the sample Tensor and $w_N$ and $w_D$ refer to normalizer and Dice factors discussed in Sec.~\ref{app:monte-carlo} of the appendix.
  }
  \label{fig:rewrite-rules-approx}
\end{figure*}

Code in user-facing models often contains a mixture of raw tensors and funsors.
We use operator overloading and multiple dispatch to handle mixtures of raw tensors and funsors.
Further, we abstract operators into classes such as Unary, Binary, and Associative (a subclass of Binary), allowing us to write individual rules that can each handle large classes of operations.

Inference code often requires patterns too deep for the minimal syntax described in Section~\ref{sec:syntax}.
Thus our implementation adds syntax for a number of compound funsors.
For example the slicing operations in Algorithm~\ref{sec:parallel-scan}
$$
f_e \from f[s_e,\, t\subs(0,2,4,6,..., 2\lfloor T/2 \rfloor-2)]
$$
are represented as symbolic Slice funsors that are equivalent to Tensor funsors but with an extra rule allowing zero-copy substitution.
Similarly concatenation in that algorithm is represented by a Cat funsor.


\subsubsection{Affine pattern matching}

Both Tensor funsors and Gaussian funsors are closed under a large class of substitutions.
Tensor funsors are closed under substitution of Tensors, thereby permitting batched computation in the presence of delayed discrete random variables.
Gaussians are closed under substitution of \emph{affine funsors} of real-tensor valued variables.

We recognize affine funsors using a two-step process:
first we specify a sound but incomplete algorithm to decide whether a pattern is affine;
and then we determine affine coefficients by substituting $1+n$ different grounding substitutions (or ``probes'') into the matched funsor: one probe to determine the constant offset, and one batched probe for each of the $n$ real-tensor valued free variables in the matched funsor.
The resulting affine funsor is represented in a canonical form using \texttt{einsum}, similar to the approach of \citep{hoffman2018autoconj}.

\subsection{Monte Carlo gradient estimation} \label{app:monte-carlo}

Stochastic gradient estimation is a fundamental computation in black-box variational inference \citep{kingma2013auto,rezende2014stochastic,ranganath2014black}, aiming to produce unbiased estimates of the gradient of a loss function w.r.t.~parameters, in the presence of integrals and sums approximated by Monte Carlo sampling.
Two broad approaches include: \emph{i}) constructing a surrogate loss function
(a secondary compute graph)
whose expected gradient matches the gradient of the expected loss \citep{schulman2015gradient}; and \emph{ii}) multiplying each stochastic choice by a differentiable ``DiCE'' factor to ensure the original compute graph is differentiable \citep{foerster2018dice}.

In our approach to stochastic gradient estimation, an approximate \textsc{MonteCarlo} interpretation stochastically rewrites one funsor (a deterministic but possibly non-analytic compute graph) to a more tractable funsor.
The key rewrite rules of the \textsc{MonteCarlo} interpretation are shown in Fig.~\ref{fig:rewrite-rules-approx}.
This allows rewriting to evaluate analytic integrals and drop zero-expectation terms before sampling.
The \textsc{MonteCarlo} interpretation rewrites Tensor and Gaussian funsors to Tensor-weighted Delta funsors that match in expectation at all derivatives.
Continuous samples are reparameterized (and hence differentiable) and weighted by a normalizer Tensor.
Discrete samples are non-differentiable, and are weighted by a normalizer and a differentiable DiCE factor.

\subsection{Atomic term data handling}

Tensor, Gaussian and Delta terms have underlying numerical data attached to them
in the form of one or more PyTorch tensors.
Our semantics and implementation rely heavily on efficient broadcasting
as implemented in PyTorch and other tensor libraries,
but because free variables have names and type contexts are unordered,
atomic terms and operations that interact with their data are responsible for maintaining an alignment between the data tensor shape and the type contexts.

In our implementation, this is accomplished by having atomic terms maintain an ordering in their type contexts (which are typically referred to as ``inputs'' in our code).
This ordering is an implementation detail and does not affect our semantics; we use it to associate individual free variables with individual dimensions of the data tensor(s).
To make this more concrete, we include the actual code for the eager binary product of two Tensor funsors from our implementation in \texttt{eager\_binary\_tensor} in Fig.~\ref{fig:example-rules}.
The \texttt{align\_tensors} function uses the ordered type contexts to permute and reshape the underlying data tensors \texttt{x = lhs.data} and \texttt{y = rhs.data} so that the resulting data tensor computed by \texttt{op(x, y)} using PyTorch's broadcasting behavior can be wrapped with a Tensor funsor with the correct type context.

\section{Details of example programs}
\label{sec:appexamples}

\subsection{Gaussian mixture model} \label{sec:gmm-code-app}

Fig.~\ref{ex:syntax} of the main text provides a complete example of a probabilistic computation in our language as described in the paper.
To illustrate the close correspondence between our paper and our actual implementation, we include in Fig.~\ref{ex:syntax-runnable} of the appendix a Python version of this example.

\begin{figure*}[ht!]
\begin{lstlisting}[language=Python, frame=tb]
from collections import OrderedDict

import torch

import funsor
from funsor.domains import bint, reals
from funsor.gaussian import Gaussian
from funsor.ops import add, logaddexp
from funsor.torch import Tensor
from funsor.terms import Variable

def gmm(x, i_z, prec_z, i_x, prec_x):
    x = Tensor(x, OrderedDict(), 'real')
    z = Variable("z", reals(2, 3))
    c = Variable("c", bint(2))
    j = Variable("j", bint(50))
    logp_c = Tensor(torch.tensor([0.5, 0.5]).log(), OrderedDict(c=bint(2)))
    logp_z = Gaussian(i_z, prec_z, OrderedDict(z=reals(3)))
    logp_xc = Gaussian(i_x, prec_x, OrderedDict(c=bint(2), z=reals(3), x=reals(3)))

    logp_x = logp_c + logp_xc(z=z[c], x=x[j])
    logp_x = logp_x.reduce(logaddexp, "c").reduce(add, "j")
    logp_x = logp_x + logp_z(z=z[c]).reduce(add, "c")
    logp_x = logp_x.reduce(logaddexp, "z")
    return logp_x
\end{lstlisting}
\caption{Funsor code for computation of the Gaussian mixture model likelihood in Fig.~\ref{ex:syntax}. The code is very similar to Fig.~\ref{ex:syntax}, except that we work with log-probabilities rather than probabilities for numerical stability and represent the plated product and marginalization terms with a single \texttt{reduce} function parametrized by an operation.}
  \label{ex:syntax-runnable}
\end{figure*}

\subsection{Maximum marginal likelihood} \label{sec:map-code-app}

Fig.~\ref{fig:map-code} in Sec.~\ref{sec:ppl} of the main text illustrates a funsor computation for maximum marginal likelihood inference in a simple generative model.
Inference steps on the right are triggered by execution of each line of model code on the left.
On line 1 the joint distribution is initialized to the trivial normalized distribution.
On line 2 a delayed sample statement triggers creation of a Variable funsor $z$ in the model code and accumulation of an unevaluated factor distribution $p_z=\op{Gaussian}((v\oftype\mathbb R),i_z,\precision_z)[v\subs z]$ in inference code.
(We assume by convention distributions like $p_z$, $p_{x|z}$, $q_{z|x}$ name their variate $v$ and parameter, if any, $\theta$.)
On line 3 a nonlinear function is lazily applied to $z$ creating a lazy funsor expression $y=\widehat\exp(z)$.
On line 4 a distribution is conditioned on ground data $x$, triggering accumulation of a factor $p_{x|z}=\op{Gaussian}((v\oftype\mathbb R, \theta\oftype\mathbb R),i_x,\precision_x)[\theta \subs y, v\subs x]$ with free variable $z$ (because $x$ is ground and $y$ has free variable $z$).
Model termination on line 5 triggers marginalization of the $z$ variable, which can be performed either exactly by pattern matching or approximately by Monte Carlo sampling.
The resulting objective is differentiable with respect to any parameters.
Optimization is achieved by repeatedly executing model code, accumulating factors, and marginalizing over $z$, using PyTorch autograd to compute gradients of the objective with respect to the model parameters, and updating the parameter values with a stochastic gradient-based optimizer.

\subsection{Variational inference} \label{sec:elbo-code-app}

Fig.~\ref{fig:elbo-code} in Sec.~\ref{sec:ppl} of the main text illustrates a typical funsor computation for variational inference, where a data-dependent variational distribution $q_{z|x}=\op{Gaussian}((v\oftype\mathbb R, \theta\oftype\mathbb R),i_q, \precision_q)[v \subs \op{Variable}(z\oftype\mathbb R), \theta \subs x]$ is fit to data.
Lines 1--6 execute delayed sample statements in the model code, and accumulate distributions $p$ and $q$ with a single free variable $z$ in inference code.
Line 7 combines $p$ and $q$ to compute the ELBO, which can be performed either exactly by pattern matching or approximately by Monte Carlo sampling $z$ from $q$.
As in the previous example, optimization is achieved by repeatedly executing model code, accumulating factors, integrating over $z$, using PyTorch autograd to compute gradients of the objective with respect to the model parameters, and updating the parameter values with a stochastic gradient-based optimizer.

This example is not terribly interesting as written, due to space constraints in the main text.
Readers looking for a more complete variational inference example may wish to examine the details of the neural variational Kalman filter described in Sec.~\ref{sec:bart} of the main text and Sec.\ref{sec:appbart} of the appendix.

\section{Experimental details}
\label{sec:appexp}

\subsection{Discrete Factor Graphs}
\label{sec:appecohmm}

In this section we describe in more detail the experiments of Sec.~\ref{sec:ecohmm} of the main text. Code for these experiments may be found in the \texttt{examples/mixed\_hmm/} directory of our source code.

\subsubsection{Dataset details}
As in \cite{obermeyer2019tensor},
we downloaded the data from \texttt{momentuHMM} \citep{mcclintock2018momentuhmm}, an R package for analyzing animal movement data with generalized hidden Markov models. 
The raw datapoints are in the form of irregularly sampled time series (datapoints separated by 5-15 minutes on average) of GPS coordinates and diving activity
for each individual in the colony (10 males and 7 females) over the course of a single day
recorded by lightweight tracking devices physically attached to each animal by researchers.
We used the \texttt{momentuHMM} harbour seal example\footnote{\url{https://github.com/bmcclintock/momentuHMM/blob/master/vignettes/harbourSealExample.R}} preprocessing code
(whose functionality is described in detail in section 3.7 of \citep{mcclintock2018momentuhmm})
to independently convert the raw data for each individual into smoothed, temporally regular time series of step sizes, turn angles, and diving activity, saving the results and using them for our population-level analysis.

\subsubsection{Model details}

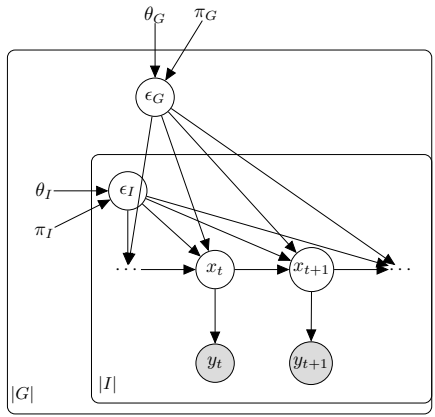
\begin{figure}[ht!]
\begin{center}
\resizebox {.7\columnwidth} {!} {
\begin{tikzpicture}
\definecolor{mygrey}{RGB}{220,220,220}

\node[latent, fill=mygrey](Y1) {$y_t$};
\node[latent, fill=mygrey, right=of Y1](Y2) {$y_{t+1}$};
\node[latent, above=of Y1](Z1) {$x_{t}$};
\node[latent, above=of Y2, right=of Z1](Z2) {$x_{t+1}$};
\node[const, left=of Z1](Z0) {$\cdots$};
\node[const, right=of Z2](Z3) {$\cdots$};

\edge {Z0} {Z1}
\edge {Z1} {Z2,Y1}
\edge {Z2} {Y2,Z3}

\node[latent, above=of Z0](EI) {$\epsilon_I$};
\node[const, left=of EI](TI) {$\theta_I$};
\node[const, below=0.5cm of TI](PI) {$\pi_I$};

\edge {EI} {Z0,Z1,Z2,Z3};
\edge {TI,PI} {EI}

\node[latent, left=of EI, above=of EI, xshift=0.5cm](EG) {$\epsilon_G$};
\node[const, above=of EG](TG) {$\theta_G$};
\node[const, right=0.5cm of TG](PG) {$\pi_G$};

\edge {EG} {Z0,Z1,Z2,Z3};
\edge {TG,PG} {EG}

%

\plate [inner xsep=0.3cm, inner ysep=0.2cm, yshift=0.1cm] {a}{(Z0)(Z3)(Z1)(Y1)(Z2)(Y2)(EI)}{}; 
\node [left plate caption=a-wrap, xshift=0.5em, yshift=0.5em]{$|I|$};

\plate [inner sep=0.4cm, yshift=0.1cm, xshift=-0.1cm] {b}{(Z0)(Z3)(Z1)(Y1)(Z2)(Y2)(EG)(TI)}{}; 
\node [left plate caption=b-wrap] {$|G|$};
\end{tikzpicture}
}
\end{center}
\caption{
A single state transition in the hierarchical mixed-effect hidden Markov model used in our experiments in Section \ref{sec:ecohmm}. $\theta$s and $\pi$s are learnable parameters.
}
\label{fig:mixed-hmm-model}
\end{figure}

Our models are discrete hidden Markov models whose state transition distribution is specified by a hierarchical generalized linear mixed model.
At each timestep $t$, for each individual trajectory $b \in I$ in each group $a \in G$, we have
\begin{multline*}
\text{logit}(p(x^{(t)}_{ab} = \text{state } i \mid x^{(t-1)}_{ab} = \text{state } j)) = \\
\big( \epsilon^\intercal_{I,ab} \theta_{1} + \epsilon^\intercal_{G,a} \theta_{2} \big)_{ij}
\end{multline*}
where $a,b$ correspond to plate indices, $\epsilon$s are independent discrete random variables, and $\theta$s are parameter vectors.
See Fig.~\ref{fig:mixed-hmm-model} for the corresponding plate diagram.

The values of the independent random variable $\epsilon_I$ and $\epsilon_G$ are each sampled from a set of three possible values shared across the individual and group plates, respectively.
That is, for each individual trajectory $b \in I$ in each group $a \in G$, we sample single random effect values for an entire trajectory:
\begin{align*}
\epsilon_{G,a} &\sim \text{Categorical}(\pi_G) \\
\epsilon_{I,ab} &\sim \text{Categorical}(\pi_{I,a}) \\
\end{align*}

Observations $y^{(t)}$ are represented as sequences of real-valued step lengths, modelled by a zero-inflated Gamma distribution, 
turn angles, modelled by a von Mises distribution,
and intensity of diving activity between successive locations, modelled with a zero-inflated Beta distribution following \citep{mcclintock2018momentuhmm,obermeyer2019tensor}.
Each likelihood component has a global learnable parameter for each possible value of $x$.
We grouped animals by sex and implemented versions of this model with no random effects, 
and with random effects present at the group, individual, or both group and individual levels.

\subsubsection{Training and implementation details}

We performed batch gradient descent with the Adam optimizer with initial learning rate $0.1$ and default momentum hyperparameters \citep{kingma2014adam}.
We annealed the learning rate once by a factor of $0.1$ when the training loss stopped decreasing,
ultimately training the models for $2000$ epochs with $10$ restarts from random initializations.
The number of random effect parameter values was taken from \citep{mcclintock2018momentuhmm}.

Our models are implemented in Python on top of PyTorch, which we use for automatic differentiation.
To compute the wall clock times in Table \ref{tbl:ecohmm}, 
we evaluated the marginal likelihood of our models on the full preprocessed harbour seal dataset and used PyTorch's reverse-mode automatic differentiation to compute gradients with respect to all trainable parameters.
All experiments were performed on an Ubuntu 18.04 workstation with a 24-core Intel Xeon processor, 64GB of RAM, and two NVIDIA RTX 2080 GPUs.
In all cases, the majority of the time is spent in the PyTorch backward pass,
which is independent of any pure Python overhead in our impementation of the forward pass.

\subsection{Additional performance experiments}
\label{app:ecohmm-perf}

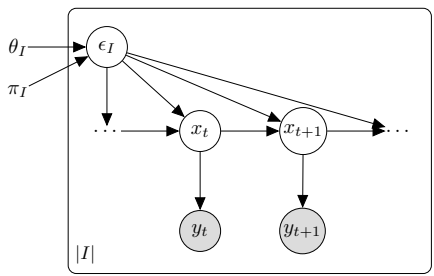
\begin{figure}[ht!]
\begin{center}
\resizebox {.7\columnwidth} {!} {
\begin{tikzpicture}
\definecolor{mygrey}{RGB}{220,220,220}

\node[latent, fill=mygrey](Y1) {$y_t$};
\node[latent, fill=mygrey, right=of Y1](Y2) {$y_{t+1}$};
\node[latent, above=of Y1](Z1) {$x_{t}$};
\node[latent, above=of Y2, right=of Z1](Z2) {$x_{t+1}$};
\node[const, left=of Z1](Z0) {$\cdots$};
\node[const, right=of Z2](Z3) {$\cdots$};

\edge {Z0} {Z1}
\edge {Z1} {Z2,Y1}
\edge {Z2} {Y2,Z3}

\node[latent, above=of Z0](EI) {$\epsilon_I$};
\node[const, left=of EI](TI) {$\theta_I$};
\node[const, below=0.5cm of TI](PI) {$\pi_I$};

\edge {EI} {Z0,Z1,Z2,Z3};
\edge {TI,PI} {EI}

\plate [inner xsep=0.3cm, inner ysep=0.2cm, yshift=0.1cm] {a}{(Z0)(Z3)(Z1)(Y1)(Z2)(Y2)(EI)}{}; 
\node [left plate caption=a-wrap, xshift=0.5em, yshift=0.5em]{$|I|$};

\end{tikzpicture}
}
\end{center}
\caption{
    A single state transition in the simplified hidden Markov model used to demonstrate the scalability of funsor variable elimination using the \textsc{MarkovProduct} operation described in Section \ref{sec:ve}.
}
\label{fig:mixed-hmm-model-perf}
\end{figure}

We systematically evaluate the parallel scaling of our algorithms using a simplified version (Figure \ref{fig:mixed-hmm-model-perf}) of the mixed-effect hidden Markov models used in Section \ref{sec:ecohmm}. 
These models have no group-level random effects and a single categorical likelihood per timestep.

For many combinations of plate size $I$ and chain length $T$,
we generate appropriately sized fake datasets and measure the time taken by tensor variable elimination \cite{obermeyer2019tensor} and funsor variable elimination described in \ref{sec:ve} to compute gradients of transition parameters $\theta, \pi$.
As in the previous section, all experiments were performed on an Ubuntu 18.04 workstation with a 24-core Intel Xeon processor, 64GB of RAM, and two NVIDIA RTX 2080 GPUs.

Figure \ref{fig:gradient-perf} shows
the average times across 20 trials for each $(I, T)$ combination. 
Our parallel algorithms scale nearly perfectly with chain length on a single GPU,
and typically achieve a wall-clock time speedup of 2 or more orders of magnitude over the sequential versions.

\begin{figure*}
  \begin{minipage}[b]{0.5\linewidth}
    \centering{
    \includegraphics[width=\columnwidth]{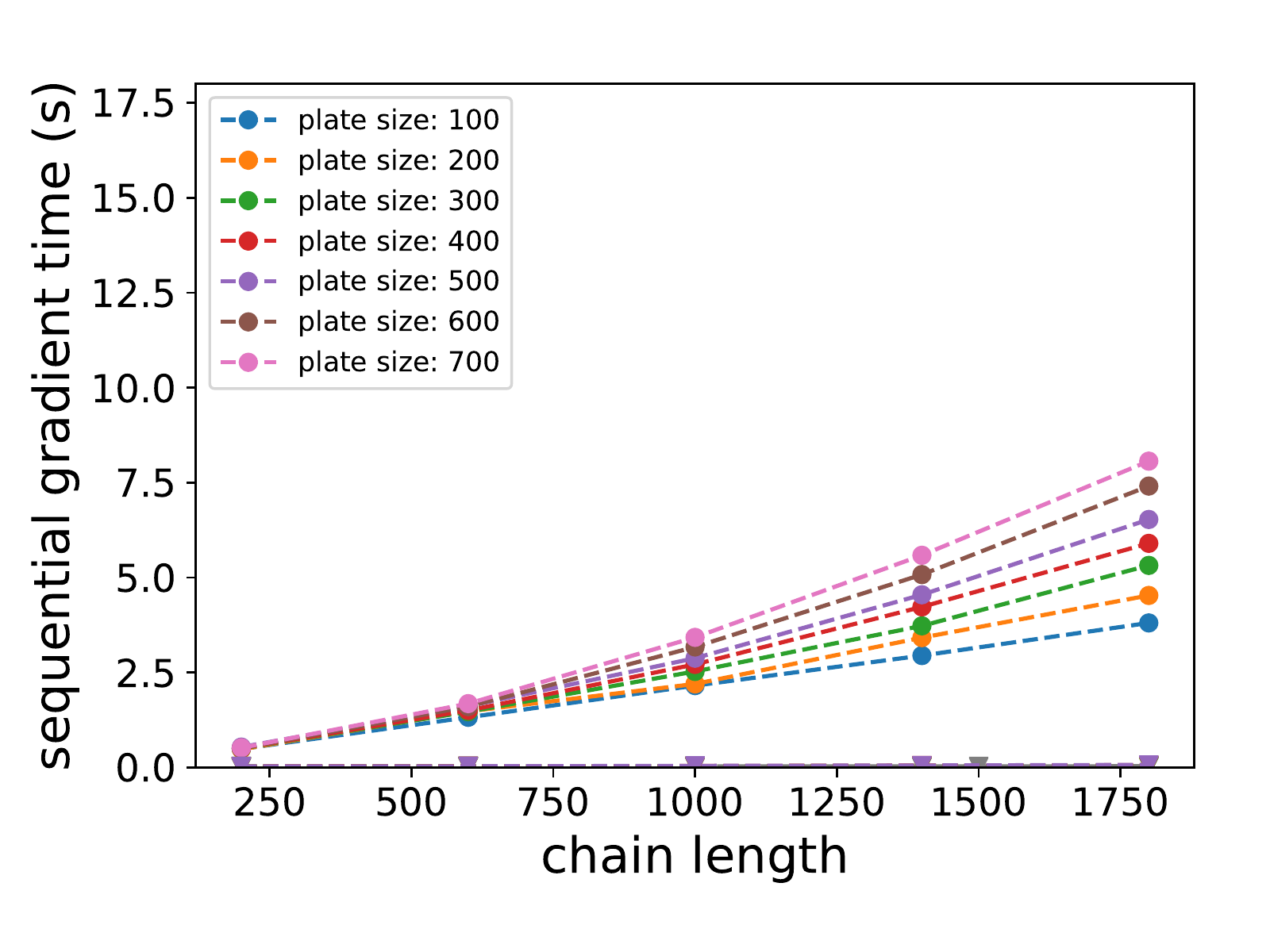}}
  \end{minipage}
  \begin{minipage}[b]{0.5\linewidth}
    \centering{
    \includegraphics[width=\columnwidth]{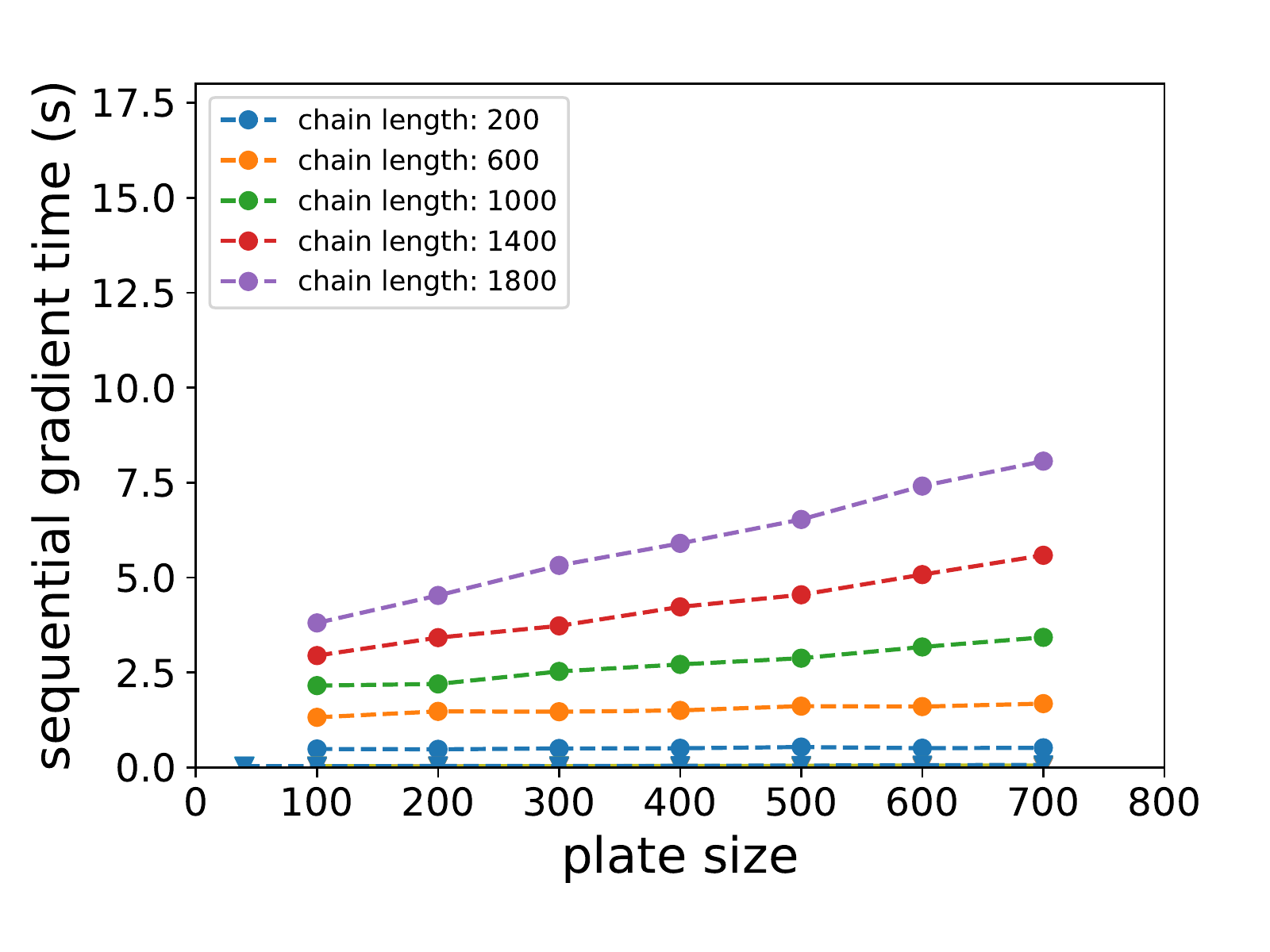}}
  \end{minipage}
  \hfill
  \begin{minipage}[b]{0.5\linewidth}
    \centering{
    \includegraphics[width=\columnwidth]{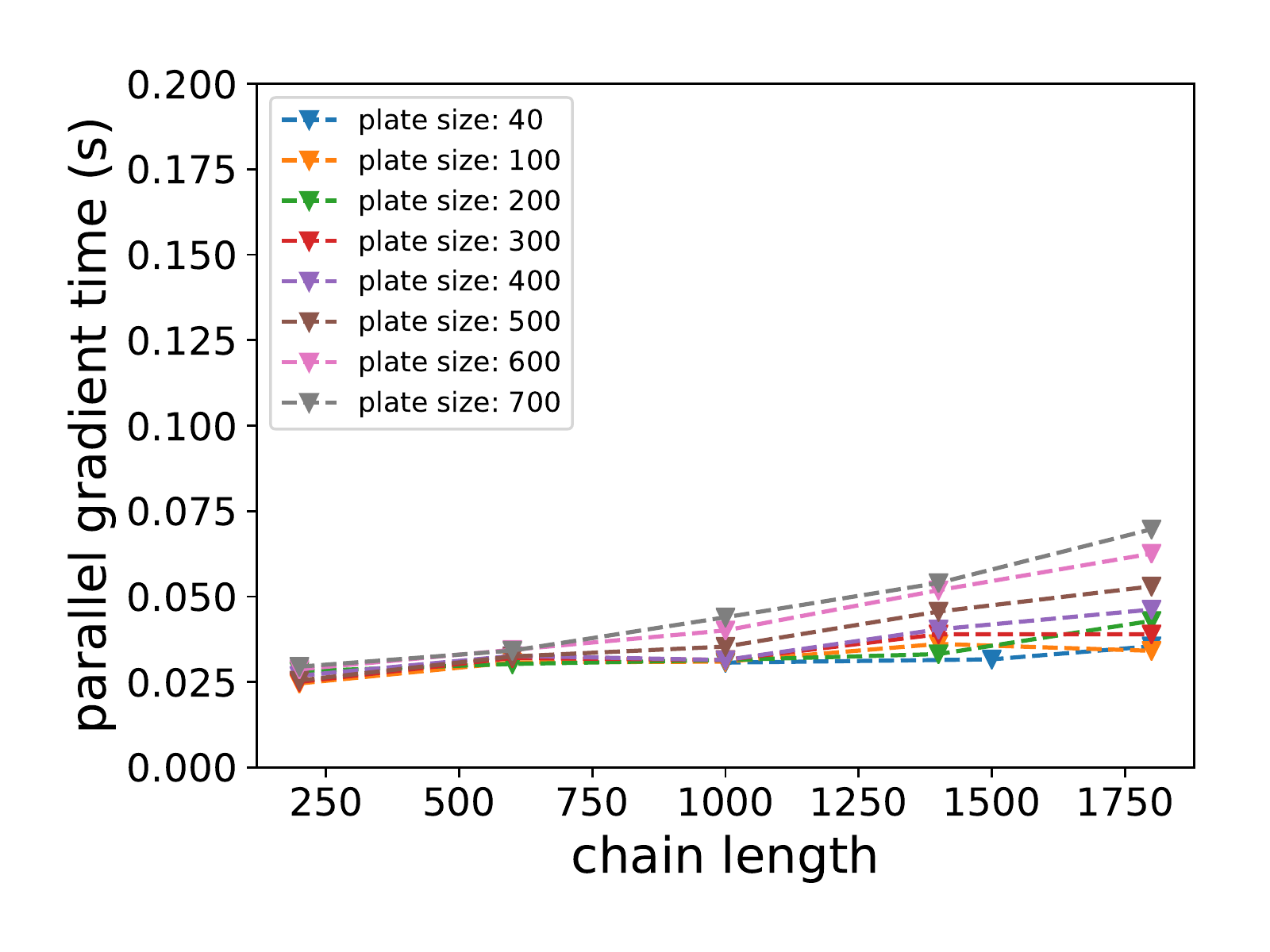}}
  \end{minipage}
  \begin{minipage}[b]{0.5\linewidth}
    \centering{
    \includegraphics[width=\columnwidth]{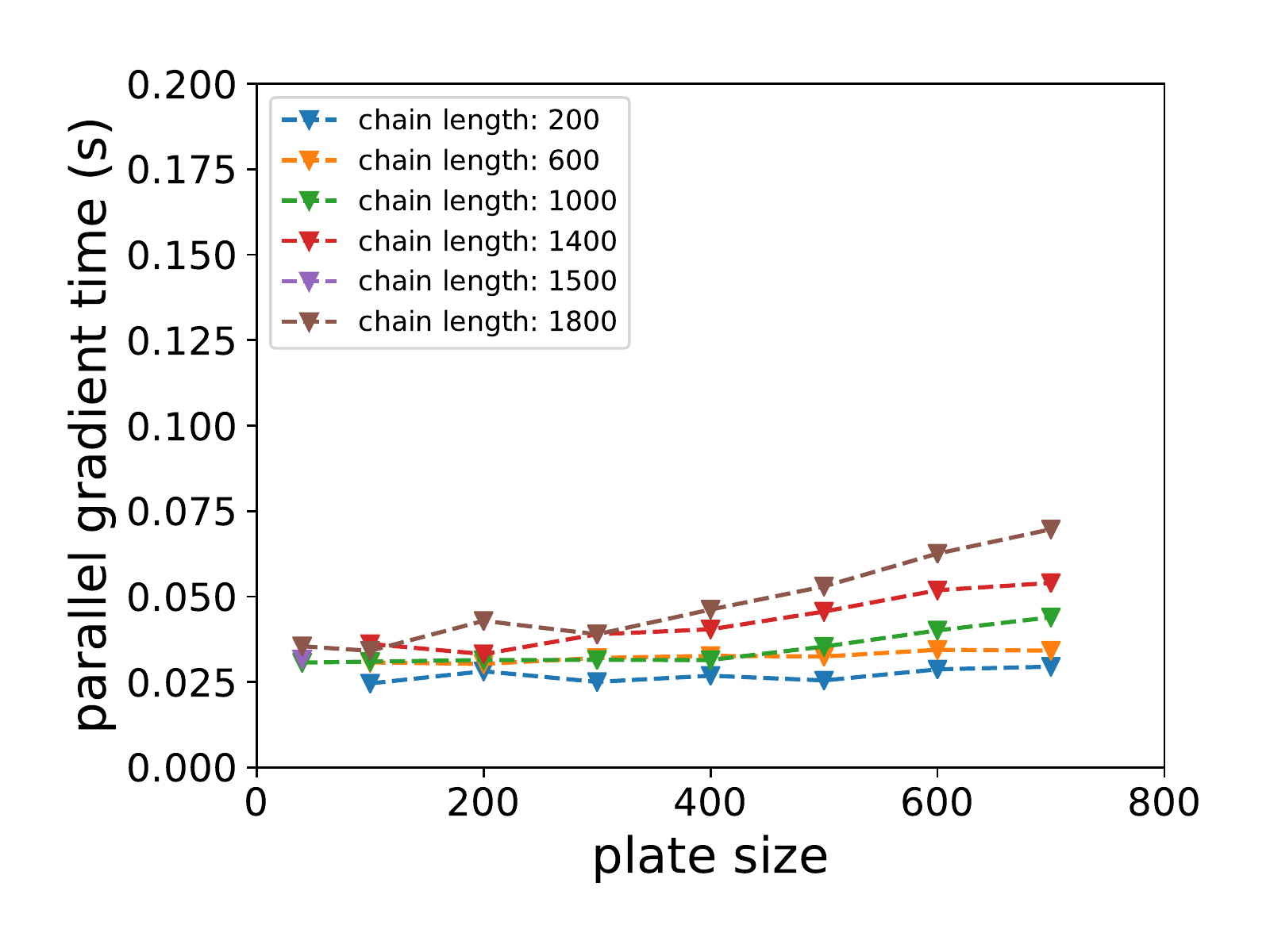}}
  \end{minipage}
  \caption{
      Visualizing average wall clock times required to compute marginal likelihoods using funsor variable elimination and gradients of the marginal likelihood for all transition parameters using PyTorch's reverse-mode automatic differentiation.
      \emph{Top left:} average computation time with sequential tensor variable elimination as a function of time series length $T$, for different fixed values of plate size $I$. The triangle-marked line at the bottom is the average time required for parallel computation, shown for scale here and in detail below.
      \emph{Bottom left:} average computation time with parallel funsor variable elimination as a function of time series length $T$, for different fixed values of plate size $I$.
      \emph{Top right:} average computation time with sequential tensor variable elimination as a function of plate size $I$, for different fixed values of time series length $T$. The triangle-marked line at the bottom is the average time required for parallel computation, shown for scale here and in detail below.
      \emph{Bottom right:} average computation time with parallel funsor variable elimination as a function of plate size $I$, for different fixed values of time series length $T$.
  }
  \label{fig:gradient-perf}
\end{figure*}

\subsection{Kalman filter with global latents}
\label{sec:appbias}

In this section we describe in more detail the experiment of Sec.~\ref{sec:bias} of the main text. Code for this experiment may be found in \texttt{examples/sensor.py} in our source code.

The underlying dynamics model used to generate the data is a 2-D linear Nearly Constant Velocity (NCV) model.
The state space model has both continuous transition states and continuous observations with additive Gaussian noise.
\begin{align*}
\nonumber
z_t = F_{t-1} z_{t-1} + q \\
x_t = H_{t-1} z_{t} + r + \beta
\end{align*}
where $F \in \mathbb{R}^{n_z \times n_z}$ is the state transition matrix, $H \in \mathbb{R}^{n_z \times n_x}$ is the state-to-observation matrix, $q, r$ are independent Gaussians
with covariances $Q, R$ respectively, and $\beta$ is the joint bias distribution, a zero mean Gaussian with learnable covariance $B$.
We can now write the conditional probabilities as follows:
\begin{align*}
    p(z_t | z_{t-1}) = N(F_{t} z_{t-1}, Q)\\
    p(x_t | z_{t}) = N(H_t z_t, R + B)
\end{align*}
To compute the marginal probability, we perform a sequential version of the sum product algorithm, collapsing out previous states after each measurement update.
We learn the joint bias scale, process noise, and observation noise using gradient descent with the Adam optimizer \citep{kingma2014adam} for 50 steps with a learning rate of 0.1, and beta parameters $(0.5, 0.8)$.


\begin{figure}
\begin{lstlisting}[language=Python, frame=tb]
curr = "state_init"
log_p = init_dist(state=curr)
log_p += bias_dist
for time, obs in enumerate(track):
    # update previous and current states
    prev, curr = curr, f"state_{time}"
    # add the transition dynamics
    log_p += trans_dist(prev=prev, curr=curr)
    # add observation noise 
    log_p += observation_dist(state=curr, obs=obs)
    # collapse out the previous state
    log_p = log_p.reduce(ops.logaddexp, prev)
# marginalize out remaining latent variables
log_p = log_p.reduce(ops.logaddexp)
\end{lstlisting}
  \caption{
    Python code using the funsor library implementing the sum product algorithm to perform inference on the biased state space model model described in section ~\ref{sec:bias}.
    Note that strings are automatically coerced to Variable funsors on substitution, as in \lstinline$init_dist(state="state_init")$.
  }
  \label{fig:appbias-code}
\end{figure}

\subsection{Switching Linear Dynamical System}
\label{sec:appslds}

In this section we describe in more detail the experiment of Sec.~\ref{sec:slds} of the main text. Code for this experiment may be found in \texttt{examples/eeg\_slds.py} in our source code.

The joint probability $p(y_{1:T}, s_{1:T}, x_{1:T})$ of model variant SLDS-I is given by
\begin{equation}
\nonumber
\prod_{t=1}^T  p(s_t | s_{t-1}) \NN(x_t | A^{s_t} x_{t-1}, \sigma_{\rm trans}^{s_t}) \NN(y_t | B x_t, \sigma_{\rm obs})
\end{equation}
where $A^{s_t}$ is a state-dependent transition matrix, $\sigma_{\rm trans}^{s_t}$ is a state-dependent
diagonal transition noise matrix, B is a state-independent observation matrix, and $\sigma_{\rm obs}$ is a state-independent
diagonal observation noise matrix. Similarly, the joint probability of variant SLDS-II is given by:
\begin{equation}
\nonumber
    \prod_{t=1}^T  p(s_t | s_{t-1}) \NN(x_t | A x_{t-1}, \sigma_{\rm trans}) \NN(y_t | B^{s_t} x_t, \sigma_{\rm obs}^{s_t})
\end{equation}
where now $A$ and $\sigma_{\rm trans}$ are state-independent and $B^{s_t}$ and $\sigma_{\rm obs}^{s_t}$ are state-dependent.
Finally, the joint probability of variant SLDS-III is given by
\begin{equation}
\nonumber
    \prod_{t=1}^T  p(s_t | s_{t-1}) \NN(x_t | A^{s_t} x_{t-1}, \sigma_{\rm trans}^{s_t}) \NN(y_t | B^{s_t} x_t, \sigma_{\rm obs}^{s_t})
\end{equation}
where now both the transition and emission probabilities are state-dependent. 
In all our experiments we use $K=2$ switching states and set the dimension of the continous state to $\rm{dim}(x_t)=5$.

To compute the log marginal likelihood used in training we use a moment-matching approximation with a window length
of $L$, see Ex.~\ref{ex:slds}. During prediction and smoothing we use $L=1$.

The raw dataset has $T=14980$ timesteps, which we subsample by a factor of 20, yielding a
dataset with $T=749$. We use the first 400 timesteps for training. Of the remaining 349 
timesteps, we use random subsets of size 149 and 200 for validation and testing, respectively.
In particular we use the validation set to choose learning hyperparameters and determine early stopping for
gradient ascent. The 14-dimensional outputs $\{y_{t}\}$ are normalized to have zero mean and unit variance.

We use the Adam optimizer for training \citep{kingma2014adam}. We train for up to 250 gradient steps and decay the learning rate exponentially.
We use the validation set to do a hyperparameter search over the exponential decay factor $\gamma$ and the momentum parameter $\beta_1$.
For each hyperparameter setting we do 7 independent runs with different random number seeds for parameter initialization. 
We then report results on the test set.

\begin{figure*}[ht!]
\begin{lstlisting}[language=Python, frame=tb]
# returns the marginal log probability of the observed data. we use an interpretation decorator to 
# signal to funsor that all reduce operations should be done using a moment-matching approximation.
#
# inputs:
# observations (torch.Tensor of shape (T, obs_dim))
# trans_probs, x_init_dist, x_trans_dist, y_dist (funsors)
@funsor.interpreter.interpretation(funsor.terms.moment_matching)
def marginal_log_prob(observations, trans_probs, x_init_dist, x_trans_dist, y_dist,
                      L=2, num_components=2, hidden_dim=5):
    log_prob = funsor.Number(0.)
    s_vars = {-1: funsor.Tensor(torch.tensor(0), dtype=num_components)}
    x_vars = {}

   for t, y in enumerate(observations):
        s_vars[t] = funsor.Variable(f"s_{t}", funsor.bint(num_components))
        x_vars[t] = funsor.Variable(f"x_{t}", funsor.reals(hidden_dim))

        # incorporate discrete switching probability p(s_t | s_{t-1})
        log_prob += dist.Categorical(trans_probs(s=s_vars[t - 1]), value=s_vars[t])

        # incorporate continuous transition probability p(x_t | x_{t-1}, s_t)
        if t == 0:
            log_prob += x_init_dist(value=x_vars[t])
        else:
            log_prob += x_trans_dist(s=s_vars[t], x=x_vars[t - 1], y=x_vars[t])

        # do a moment-matching reduction of latent variables from L time steps in the past 
        # [i.e. we retain a running (L+1)-length window of latent variables throughout the for loop]
        if t > L - 1:
            log_prob = log_prob.reduce(ops.logaddexp, {s_vars[t - L].name, x_vars[t - L].name})

        # incorporate observation probability p(y_t | x_t, s_t)
        log_prob += y_dist(s=s_vars[t], x=x_vars[t], y=y)

    T = data.shape[0]
    for t in range(L):
        log_prob = log_prob.reduce(ops.logaddexp, {s_vars[T - L + t].name, x_vars[T - L + t].name})
            
    return log_prob
\end{lstlisting}
\caption{\texttt{funsor} code for the computation of the log marginal probability $\log p(y_{1:T})$ for the SLDS model
in Sec.~\ref{sec:slds}.}
\label{ex:slds}
\end{figure*}

\subsection{Neural variational Kalman filter}
\label{sec:appbart}

In this section we describe in more detail the experiment of Sec.~\ref{sec:bart} of the main text. Since this experiment is more complex, in addition to full source code in a separate file we also include isolated Funsor code in Figs.~\ref{fig:appbart-model}-\ref{fig:appbart-elbo} of the appendix.

We examine data publicly available at \url{https://www.bart.gov/about/reports/ridership}, containing hourly ridership counts between every pair of Bay Area Rapid Transit train stations for the years 2011-2018.
Our objective is to jointly forecast all station-station pairs such that users can aggregate these forecasts as desired, e.g. rides between a given pair of stations, or all arrivals-to and departures-from a given station, as in Figures~\ref{fig:appbart-forecast}.

\begin{figure*}
  \centering{
    \includegraphics[scale=0.6]{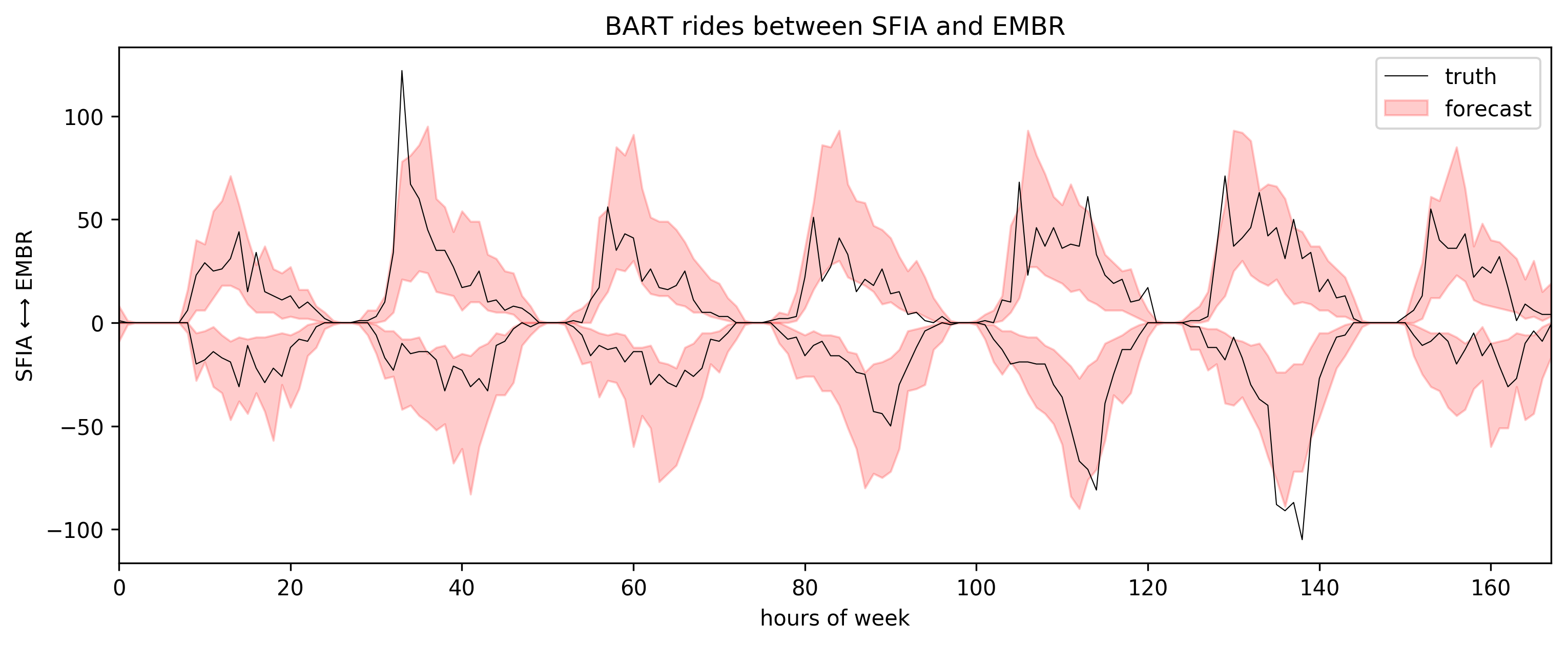}
    \includegraphics[scale=0.6]{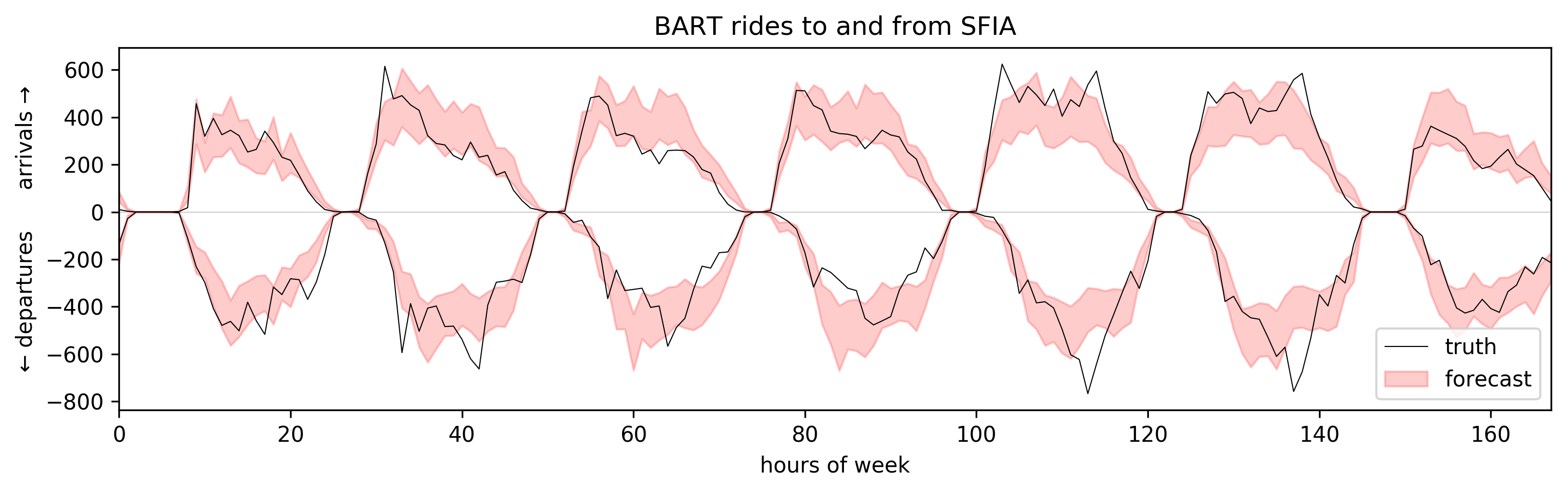}
    \includegraphics[scale=0.6]{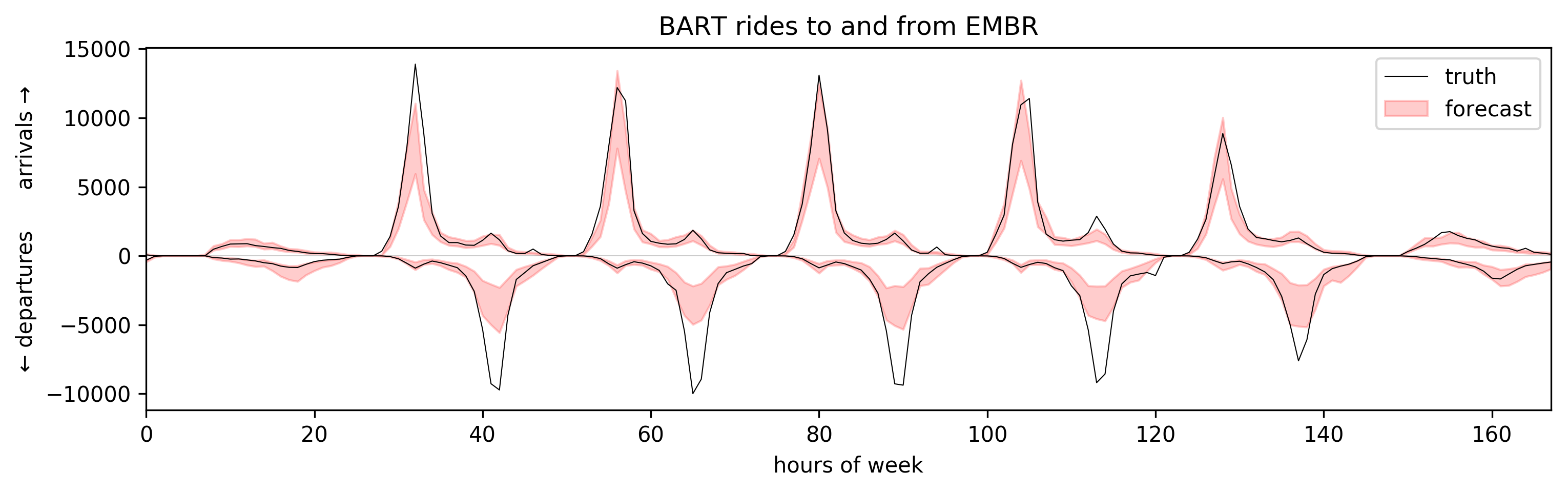}
  }
  \caption{
    One week of forecasted (pink region) and true (black line) traffic for the San Francisco international airport station and Embarcadero station.
    Forecast regions are 10\% and 90\% percentiles, and should bound the truth roughly 80\% of the time.
    Note total arrivals-to and departures-from any one station are much larger than traffic between a single pair of stations.
    See Sec.~\ref{sec:appbart} for details.
  }
  \label{fig:appbart-forecast}
\end{figure*}

Our generative model and collapsed variational inference model are shown in Figure~\ref{fig:appbart-model} and Figure~\ref{fig:appbart-guide}, respectively.
Our model neural network is a multilayer perceptron of the form linear-sigmoid-linear, whose output we split into (i) a gate logit which is mapped to a probability $p$ via a sigmoid funnction, and (ii) a Poisson rate which is mapped to a bounded positive number $\lambda$ via a bounded exponential function (combining an affine transform and a sigmoid).
Our guide neural network is a multilayer perceptron of the form linear-sigmoid-linear and with middle layer tuned to the same low-dimension as the model state space (we examine sizes $\{2,4,8\}$).
Both neural networks operate independently over each time step, i.e. we rely on the Gaussian state space model rather than a convolution or RNN for coupling states over time.
Our mean field inference model additionally predicts latent state $z$ from a multilayer perceptron; this neural net introduces weak time dependency in the form of a fully time-pooled layer.

\begin{figure*}[ht!]
\begin{lstlisting}[language=Python, frame=tb]
def model(features, trip_counts):
    total_hours = len(features)
    observed_hours, n, n = trip_counts.shape
    gate_rate = funsor.Variable("gate_rate_t", reals(observed_hours, 2 * n * n))["time"]

    @funsor.torch.function(reals(2 * n * n), (reals(n, n, 2), reals(n, n)))
    def unpack_gate_rate(gate_rate):
        batch_shape = gate_rate.shape[:-1]
        gate, rate = gate_rate.reshape(batch_shape + (2, n, n)).unbind(-3)
        gate = gate.sigmoid().clamp(min=0.01, max=0.99)
        rate = bounded_exp(rate, bound=1e4)
        gate = torch.stack((1 - gate, gate), dim=-1)
        return gate, rate

    # Create a Gaussian latent dynamical system.
    init_dist, trans_matrix, trans_dist, obs_matrix, obs_dist = \
        nn_dynamics(features[:observed_hours])
    init = dist_to_funsor(init_dist)(value="state")
    trans = matrix_and_mvn_to_funsor(trans_matrix, trans_dist,
                                     ("time",), "state", "state(time=1)")
    obs = matrix_and_mvn_to_funsor(obs_matrix, obs_dist,
                                   ("time",), "state(time=1)", "gate_rate")

    # Compute dynamic prior over gate_rate.
    prior = trans + obs(gate_rate=gate_rate)
    prior = MarkovProduct(ops.logaddexp, ops.add,
                          prior, "time", {"state": "state(time=1)"})
    prior += init
    prior = prior.reduce(ops.logaddexp, {"state", "state(time=1)"})

    # Compute zero-inflated Poisson likelihood.
    gate, rate = unpack_gate_rate(gate_rate)
    likelihood = fdist.Categorical(gate["origin", "destin"], value="gated")
    trip_counts = tensor_to_funsor(trip_counts, ("time", "origin", "destin"))
    likelihood += funsor.Stack("gated", (
        fdist.Poisson(rate["origin", "destin"], value=trip_counts),
        fdist.Delta(0, value=trip_counts)))
    likelihood = likelihood.reduce(ops.logaddexp, "gated")
    likelihood = likelihood.reduce(ops.add, {"time", "origin", "destin"})

    return prior + likelihood
\end{lstlisting}
  \caption{
    Funsor computation of the generative model in ridership forecasting Sec.~\ref{sec:appbart}.
    Here \lstinline$nn_dynamics(-)$ is the model's neural network.
    Note that \lstinline$@funsor.torch.function$ combines the lifting operator $\widehat f$ with variable substitution, and is used to lift a Python/PyTorch function to a funsor with one free variable for each function argument.
    Note also the use of the take(-,-) operator described in Example~\ref{ex:syntax}, used to change the shape of funsors like \lstinline$rate["origin", "destin"]$, where the symbols ``origin'' and ``destin'' are coerced to Variable funsors.
    We use helpers like \lstinline$matrix_and_mvn_to_funsor$ to convert between PyTorch distribution objects and funsors.
  }
  \label{fig:appbart-model}
\end{figure*}

\begin{figure*}[ht!]
\begin{lstlisting}[language=Python, frame=tb]
def guide(features, trip_counts):
    observed_hours = len(trip_counts)
    log_counts = trip_counts.reshape(observed_hours, -1).log1p()
    loc_scale = ((nn_diag_part * log_counts.unsqueeze(-2)).reshape(observed_hours, -1) +
                 nn_lowrank(torch.cat([features[:observed_hours], log_counts], dim=-1)))
    loc, scale = loc_scale.reshape(observed_hours, 2, -1).unbind(1)
    scale = bounded_exp(scale, bound=10.)

    # Create a diagonal normal distribution.
    diag_normal = dist.Normal(loc, scale).to_event(2)
    return dist_to_funsor(diag_normal)(value="gate_rate_t")
\end{lstlisting}
  \caption{
    Funsor computation of the collapsed variational inference model in ridership forecasting Sec.~\ref{sec:appbart}.
    Here \lstinline$nn_diag_part$ is a learnable parameter and \lstinline$nn_lowrank(-)$ is the guide's neural network.
    We use helpers like \lstinline$dist_to_funsor$ to convert between PyTorch distribution objects and funsors.
  }
  \label{fig:appbart-guide}
\end{figure*}

We train on random two-week minibatches of data.
This introduces two forms of bias which we argue are negligible.
First, subsampling a time series introduces dependency bias, however empirically data is week-to-week Markov, so that a single week of data captures all long-term effects.
Second, subsampling windows introduces a trapezoidal data weighting whereby the first two weeks and last two weeks are not uniformly sampled; however we find this negligible since we train on at least six years of data at a time.
An advantage of minibatches spanning exactly two weeks is that the cyclic hour-of-week features are evenly covered in each minibatch.

We train using an Adam optimizer \cite{kingma2014adam} with gradient clipping, learning rate that exponentially decays from $0.05$ to $0.005$ over 1000 gradient descent steps, and momentum parameters $\beta=(0.8, 0.99)$.
Our loss function is the negative ELBO, as computed in Figure~\ref{fig:appbart-elbo}.

\begin{figure}[ht!]
\begin{lstlisting}[language=Python, frame=tb]
def elbo_loss(features, counts):
    # Interpret the inference model exactly.
    q = guide(features, counts)

    # Interpret the generative model lazily.
    with interpretation(lazy):
        p = model(features, counts)
        pq = p - q

    # Monte Carlo approximate the ELBO.
    with interpretation(monte_carlo):
        elbo = funsor.Integrate(q, pq, "gate_rate_t")

    loss = -elbo
    assert isinstance(loss, funsor.Tensor)
    return loss.data
\end{lstlisting}
  \caption{Funsor computation of the ELBO for variational inference in ridership forecasting Sec.~\ref{sec:appbart}. Note that here the user interleaves three different interpretations: \textsc{Exact}, \textsc{Lazy}, and \textsc{MonteCarlo}.}
  \label{fig:appbart-elbo}
\end{figure}

\end{document}